\documentclass[journal]{IEEEtran}
\usepackage{threeparttable}
\usepackage{adjustbox}
\usepackage{multirow} 
\usepackage{booktabs}
\usepackage{amsmath,amsfonts}
\usepackage{algorithmic}
\usepackage{algorithm}
\usepackage{array}
\usepackage{color}
\usepackage{textcomp}
\usepackage{stfloats}
\usepackage{url}
\usepackage{verbatim}
\usepackage{graphicx}
\usepackage{epstopdf}
\usepackage{epsfig}
\usepackage{subfigure}
\usepackage{afterpage}
\newtheorem{theo}{Theorem}

\newtheorem{proof}{Proof}

\begin{document}

\title{One-step Multi-view Clustering With Adaptive
Low-rank Anchor-graph Learning}

\author{ 
Zhiyuan~Xue,
Ben~Yang,~\IEEEmembership{Member,~IEEE,}
Xuetao~Zhang,~\IEEEmembership{Member,~IEEE,}
Fei~Wang,~\IEEEmembership{Member,~IEEE,}
and Zhiping~Lin,~\IEEEmembership{Senior~Member,~IEEE}
\thanks{This work was supported in part by the National Natural Science Foundation of China (No. 62088102, No. 62403372), the National Postdoctoral Innovative Talents Support Program (No. BX20230283), the Shaanxi Province "Sanqin Bochuang" Talent Support Plan (No. 2024SQBC005), the Project funded by China Postdoctoral Science Foundation (No. 2023M742791), the Natural Science Basic Research Program of Shaanxi Province (No. 2023-JC-YB-486, No. 2024-JC-YBQN-0658), the Natural Science Foundation of Sichuan Province (No. 2025ZNSFSC1496), and the Fundamental Research Funds for the Central Universities (No. xzy012023137)(\textit{Corresponding Author: Xuetao Zhang})}
\thanks{Zhiyuan Xue, Ben Yang, Xuetao Zhang, and Fei Wang are with the National Key Laboratory of Human–Machine Hybrid Augmented Intelligence, the National Engineering Research Center for Visual Information and Applications, and the Institute of Artificial Intelligence and Robotics, Xi'an Jiaotong University, Xi'an 710049, China (e-mail: xuezhiyuan@stu.xjtu.edu.cn; benyang@xjtu.edu.cn; xuetaozh@xjtu.edu.cn; wfx@xjtu.edu.cn)}

\thanks{Zhiping Lin is with the School of Electrical and Electronic Engineering, Nanyang Technological University, Singapore 639798, Singapore. (e-mail: ezplin@ntu.edu.sg)}
 }

\markboth{Journal of \LaTeX\ Class Files,~Vol.~xx, No.~xx, August~20xx}%
{Shell \MakeLowercase{\textit{et al.}}: A Sample Article Using IEEEtran.cls for IEEE Journals}

\maketitle

\begin{abstract}
In light of their capability to capture structural information while reducing computing complexity, anchor graph-based multi-view clustering (AGMC) methods have attracted considerable attention in large-scale clustering problems. Nevertheless, existing AGMC methods still face the following two issues: 1) They directly embedded diverse anchor graphs into a consensus anchor graph (CAG), and hence ignore redundant information and numerous noises contained in these anchor graphs, leading to a decrease in clustering effectiveness; 2) They drop effectiveness and efficiency due to independent post-processing to acquire clustering indicators. To overcome the aforementioned issues, we deliver a novel one-step multi-view clustering method with adaptive low-rank anchor-graph learning (OMCAL). To construct a high-quality CAG, OMCAL provides a nuclear norm-based adaptive CAG learning model against information redundancy and noise interference. Then, to boost clustering effectiveness and efficiency substantially, we incorporate category indicator acquisition and CAG learning into a unified framework. Numerous studies conducted on ordinary and large-scale datasets indicate that OMCAL outperforms existing state-of-the-art methods in terms of clustering effectiveness and efficiency.

\end{abstract}


\section{Introduction}

\IEEEPARstart{T}{he} rapid development of multimedia technology and information technology has led to the explosive growth of multi-view data. In the realm of multi-view clustering \cite{chao2021survey,fang2023comprehensive}, graph-based multi-view clustering (GMC) \cite{zhan2017graph,li2021consensus,wang2022towards,fu2022unified} methods have garnered significant attention for their capacity to capture rich structural information within the given data. Typically, GMC captures this structural information of the data by evaluating the similarity among all samples. For instance, \cite{nie2017multi}  directly assesses the similarity between each pair of samples across all views, \cite{zhu2018one} focuses on finding a low-dimensional space that reduces noise to enhance similarity measurement between pairs of samples, and \cite{long2023multi} employs subspace theory to construct a graph structure that reflects the similarity between each pair of samples, etc. Although GMC demonstrates excellent clustering performance, its quadratic computational complexity, stemming from the similarity calculation among all samples, limits its scalability to large-scale data. To overcome this limitation, anchor graph-based multi-view clustering (AGMC) \cite{guo2019anchors,yang2022efficientrobust,yang2024discrete,zhao2024anchor} methods have rapidly advanced, owing to their capability to capture structural information while mitigating computational complexity.

AGMC first identifies representative anchors from multi-view data and captures structural information by measuring the similarity among samples and these anchors. Since AGMC selects a significantly smaller number of anchors compared to the total sample size, building the anchor graph entails only linear computational complexity. However, existing AGMC methods face two main issues: 1) The prevailing methods directly embed the anchor graphs of different views into a consensus anchor graph (CAG), ignoring redundant information
and numerous noises contained in these anchor graphs; 2) The acquisition of clustering indices in current AGMC methods necessitates a separate post-processing step. This separation between consensus anchor graph construction and clustering index acquisition introduces additional computational complexity and leads to information loss.

To fuse information from multi-view data, early AGMC methods assign varying weights to each view, treating the final consensus anchor graph as the weighted sum of these views, which partially distinguishes their contributions \cite{yang2020fast,wang2021semi,yang2024fast}. However, multi-view data offers descriptions of the same samples from different perspectives, with each view containing a significant amount of unique but redundant information. For example, when describing the identity of the same person using fingerprint and iris views, each view contains information unique to its respective modality. Therefore, constructing anchor graphs independently for each view and directly embedding them leads to information redundancy. In recent years, numerous studies have focused on obtaining high-quality consensus anchor graphs. Some methods enhance the embedding space of the consensus anchor graph by exploring more representative anchor vector spaces \cite{wang2022align, qin2023edmc, yu2023multi,li2023incomplete}. While these approaches have mitigated the impact of redundant information in each view to some extent, they do not directly eliminate redundancy within the consensus anchor graph itself. Other methods construct tensors for anchor graphs and employ low-rank learning strategies to enforce similarity for the same sample across all views \cite{dai2023tensorized,shu2022self,ji2023anchor}. Although these methods achieve highly consistent representations of the same sample across different views, they do not guarantee the separability of all samples in the consensus anchor graph, which hinders the improvement of clustering accuracy.

To analyze the anchor graph for clustering indicators, some methods simulate the full-sample graph by learning the embedding representation of the anchor graph \cite{li2015large,kang2020large,liu2024learn}. These approaches typically utilize singular value decomposition (SVD) to derive the anchor graph embedding representation, followed by post-processing steps like k-means to derive the final clustering indicators. These methods significantly lower the complexity of the analysis process by directly analyzing the anchor graph, with most exhibiting linear computational complexity. However, they require post-processing steps to analyze the embedding representation of the anchor graph for getting clustering indicators, which not only increases computational complexity but also leads to additional information loss. To address these issues, some methods integrate the post-processing step with anchor graph embedding learning into a single process \cite{shi2021fast,yang2022fast,yu2023sample}. The aforementioned methods approximate graph embedding by learning cluster indicator matrices. Although these methods eliminate the need for post-processing steps, they still require the optimization of additional variables, such as direct mapping or graph embedding rotation after obtaining the graph embedding, which may compromise computational efficiency.

To address the aforementioned challenges, we propose a one-step multi-view clustering method with adaptive low-rank anchor graph learning (OMCAL) method, which integrates low-rank consensus anchor graph learning and clustering indicators acquisition within a unified framework. In our framework, the consensus anchor graph learning is guided not only by the low-rank constraint but also by the clustering target. This joint optimization ensures that the learned consensus graph is both low-rank and well-aligned with the clustering structure. Specifically, OMCAL lowers the impact of redundant information and noise on the consensus structure by performing low-rank learning on the consensus anchor graph derived from all views. Simultaneously, OMCAL performs orthogonal matrix decomposition on the consensus anchor graph, constraining one of the factor matrices as non-negative maxtrix to serve as the clustering indicator matrix. In this single step, OMCAL can efficiently obtain clustering indicator without losing information. We execute a series of experiments involving various multi-view datasets,  including large-scale data, demonstrating the performance and efficiency advantages of our method. The primary contributions of this paper are outlined below:

\begin{enumerate}
\itemsep=0pt
\item 
We proposed a low-rank consensus anchor graph learning strategy to fully explore the underlying correlations among samples while taking efficiency into account.
\item 
We integrated anchor graph learning and clustering indicator acquisition within a unified framework to extract a clearer underlying clustering structure and prevent the loss of clustering information and efficiency.
\item
Leveraging the orthogonality of matrices and the properties of the matrix trace, we designed a fast optimization method for OMCAL.
\item
We conducted extensive experiments on multi-view datasets, demonstrating the effectiveness and efficiency of the proposed OMCAL. 
\end{enumerate}

The remainder of this paper is structured as follows: Section II reviews related multi-view clustering methods based on anchor graph; Section III presents preliminary knowledge and notation explanations; Section IV outlines the method proposed in this paper; Section V details the optimization method we designed for OMCAL; and Section VI presents a series of experiments and discussions. Finally, Section VII concludes this paper.

\section{Related Works}
This paper focuses primarily on two aspects: 1) obtaining a high-quality consensus anchor graph to uncover the underlying structure of the data; and 2) efficiently analyzing anchor graphs to achieve effective clustering indicator. Therefore, this section will present a succinct overview of these two areas of work.

In the context of obtaining a high-quality consensus anchor graph and exploring the underlying structure of multi-view data, some AGMC methods focus on identifying more representative anchor vector spaces \cite{wang2022align, qin2023edmc, yu2023multi,li2023incomplete}. For example, \cite{wang2022align} proposes a novel anchor graph fusion framework to align anchors of different views. \cite{qin2023edmc} leverages the cluster representation space to generate physically meaningful and informative anchors. \cite{yu2023multi} generates unified anchors for all views to improve clustering efficiency and effect. \cite{li2023incomplete} employs dual-attention and dual contrastive learning to capture view-specific anchors and model the relationship between samples and anchors.  However, information from different feature spaces contains not only the underlying clustering information but also redundant details unique to each view. Simply identifying high-quality anchor vector spaces does not directly eliminate the redundant information present in the consensus anchor graph. 
Some methods achieve a high-quality consensus graph by minimizing differences among the representations of each anchor graph \cite{dai2023tensorized,shu2022self,ji2023anchor}. For instance, \cite{dai2023tensorized} combines anchor graphs from various views into a tensor and employs its nuclear norm to constrain the approximate representation of anchor mappings across these views. \cite{shu2022self} employs the tensor Schatten p-norm to treat the weights belongs to all views distinctly, ultimately obtaining an approximate representation of the anchor graph across these views. \cite{ji2023anchor} designs an enhanced tensor rank to capture high-order similarities among views. However, most of these methods focus solely on the similarity of the same sample across views and do not address the underlying clustering structure within the consensus anchor graph. In addition, over-emphasizing cross-view consistency can lead to the loss of structural information \cite{lu2024decoupled}. As a result, they cannot guarantee the separability of all samples in the consensus anchor graph, which hinders improvements in clustering accuracy.

In the context of analyzing anchor graphs to obtain clustering indicator, many recent AGMC methods utilize the embedding representation of anchor graphs to simulate full-sample graphs. These methods exhibit linear complexity in obtaining anchor graph embedding representations, significantly expanding the applicability of AGMC to large-scale data \cite{li2015large,kang2020large,liu2024learn}. For instance, \cite{li2015large} employs local manifold fusion to integrate multi-view anchor graphs and applies k-means clustering on the singular value vectors of the consensus anchor graph to achieve clustering indicator. Similarly, \cite{kang2020large} uses a self-representation strategy to learn the anchor graph and utilizes the left singular values of the consensus anchor graph to simulate the consensus full-sample graph. Nevertheless, these methods still involve post-processing steps like k-means to derive clustering indicators, leading to information loss and increased computational demands. To address these issues, some methods integrate post-processing with anchor graph embedding learning in the same step 
 \cite{shi2021fast,yang2022fast,yu2023sample}. For instance, \cite{shi2021fast} combines anchor graph embedding with graph rotation within a unified framework to directly generate clustering metrics. \cite{yang2022fast} streamlines the learning of rotation matrices by directly mapping anchor graph embedding to indicator matrices, facilitating the extraction of final clustering metrics. Although these methods eliminate the need for post-processing steps like k-means, they still require the optimization of additional variables after obtaining graph embedding, which can compromise computational efficiency.

\section{Notations and Preliminaries}
In this section, we will introduce the notations and prerequisite knowledge used in this paper.

\subsection{Notations}
In terms of matrix and vector representation, all matrices are represented by bold uppercase letters. For example, $\mathbf{A}\in\mathbb{R}^{a\times b}$ represents a matrix with $a$ rows and $b$ columns. The $(i,j)$-th element of matrix $\mathbf{A}$ is written as $\mathbf{A}_{ij}$. The $i$-th row and $j$-th column of matrix $\mathbf{A}$ are written as $\mathbf{A}_{i \cdot}$ and $\mathbf{A}_{\cdot j}$ respectively. All vectors are written in lowercase bold letters. For instance, $\boldsymbol{a} \in \mathbb{R}^{n \times 1}$ is a column vector. The $i$-th element of vector $\boldsymbol{a}$ is denoted as $\boldsymbol{a}_{i}$. In terms of the meaning of mathematical symbols, $\rm{Tr}(\mathbf{A})$ represents the trace of matrix $\mathbf{A}$. $\left \| \mathbf{A} \right \| _{\ast } $ and $ \left \| \mathbf{A} \right \|_{\rm{F}}$ represent the nuclear norm and Frobenius norm of the matrix $\mathbf{A}$ respectively. $\mathbf{A}^{\top}$ represents the transpose of matrix $\mathbf{A}$. $\mathbf{A}^{-1}$ represents the inverse matrix of $\mathbf{A}$. $\left \| \boldsymbol{a} \right \|_{2} $ represent the 2-norm of vector $\boldsymbol{a}$.

\subsection{Preliminaries}

\subsubsection{Non-negative matrix factorization} 
Assume that a set of single-view data is $\mathbf{X}\in\mathbb{R}^{n\times d}$, where $n$ represents the number of samples and $d$ represents the number of dimensions. Non-negative Matrix Factorization (NMF) \cite{lee1999learning} intends to factor original data into the product of base matrix and weight matrix that are both non-negative. For a clustering problem, if we wish to classify $\mathbf{X}$ into $c$ classes, we can constrain base matrix as $\mathbf{G} \in \mathbb{R}^{d\times c}$ and weight matrix as $\mathbf{F} \in \mathbb{R}^{n \times c}$. The problem of NMF can be written as:
\begin{equation}
\label{eq:NMF}
 \min_{\mathbf{F}\geq 0,\mathbf{G}\geq 0}{\left\Vert \mathbf{X}-\mathbf{FG^{\top}}\right\Vert^{2}_{\rm{F}}}
\end{equation}

\subsubsection{Anchor Graph Construction}
Constructing an anchor graph involves two key processes: selecting the anchors and constructing the anchor graphs. Anchor selection methods include k-means and random sampling. In comparison, anchors obtained through k-means are more representative of the overall distribution \cite{li2015large}. Anchor graph aims to reveal the similarity among the anchors and the samples. To further thin out the anchor graph and obtain the principal component, the $k$-NN strategy has been extensively adopted for anchor graph construction recently.  Following the strategy in \cite{nie2016constrained}, a normalized anchor graph can be written as follows:
\begin{equation}
 \label{eq:construct anchor graph}
 \mathbf{S}_{ij}=\left\{\begin{matrix}
 \frac{\Phi \left(\boldsymbol{x}_{i},\boldsymbol{c}_{k+1} \right)-\Phi \left(\boldsymbol{x}_{i},\boldsymbol{c}_{j}\right)}{k\Phi \left(\boldsymbol{x}_{i},\boldsymbol{c}_{k+1} \right)- {\textstyle \sum_{h=1}^{k}}\Phi \left(\boldsymbol{x}_{i},\boldsymbol{c}_{h} \right) } & j\le k \\ \\
 0&j> k
 \end{matrix}\right.
\end{equation}
where $k$ represents the number of anchors nearest to the sample $\boldsymbol{x}_{i}$, $\boldsymbol{c}_{j}$ represent the $j$-th anchor, and $\Phi(\cdot)$ indicates the similarity measurement method.

\subsubsection{Nuclear norm} 
A low-rank matrix indicates that a matrix can be expressed by a linear combination of a limited number of principal component vectors \cite{liu2012robust,jing2023prototype}. This implies that the principal components can be further obtained by low-rank learning of the matrix. Given a matrix $\mathbf{A}$, the low-rank learning of $\mathbf{A}$ can be expressed as follows:
\begin{equation}
\label{eq:rank}
 \min_{\mathbf{A} } rank(\mathbf{A})
\end{equation}

However, optimizing the Eq. (\ref{eq:rank}) is difficult due to its non-convexity.  Consequently, numerous prior studies have transformed rank constraints into nuclear norm constraints to facilitate convex optimization \cite{recht2010guaranteed,chen2018harnessing}. The nuclear norm serves as an approximation of the rank function and is defined as follows:
\begin{equation}
 \label{eq:nulear_norm_define}
 \left \| \mathbf{A} \right \| _{\ast } = \sum_{i=1}^{h} \sigma_{i}
\end{equation}
where $h$ denotes the total number of singular value of $\mathbf{A}$, and $\sigma_{i}$ represents the $i$-th largest among them.

Due to the non-negativity of singular values, we have $\left \| \mathbf{A} \right \| _{\ast }\ge 0$. Additionally, to further filter out noise and redundant information from $\mathbf{A}$, based on the definition of Eq. (\ref{eq:nulear_norm_define}), the problem in Eq. (\ref{eq:rank}) can be relaxed as follows:
\begin{equation}
 \min_{\mathbf{S} } \Phi\left ( \mathbf{S},\mathbf{A} \right ) + \left \| \mathbf{S} \right \| _{\ast }
\end{equation}
where $\Phi\left ( \mathbf{S},\mathbf{A} \right )$ represents the distance between $\mathbf{S}$ and $\mathbf{A}$.

\section{Methodology}
In this section, we presents the proposed OMCAL model. Assume that the original multi-view anchor graphs are denoted as $\mathcal{S}=\left[\mathbf{S}^{(1)},\mathbf{S}^{(2)},\dots,\mathbf{S}^{(V)} \right]$, $\mathbf{S}^{(i)}\in\mathbb{R}^{n\times m} $. Where samples will be divided into $c$ clusters. Here, $V$ represents the number of views. We will introduce the two components of the OMCAL structure: learning a consensus graph and getting clustering indicator in the same step.

\subsection{Learning a consensus graph}
Since the number of anchors of all anchor graphs is $m$, the dimensions of all anchor graphs are equal. As previously discussed, effectively handling the information across different views is crucial for multi-view clustering. This paper uses the local manifold fusion method to assign distinct weights to each anchor graph. We aim to combine the structural information from each view to achieve a consensus joint anchor graph. The consensus anchor graph can be expressed as follows:
\begin{equation}
 \mathbf{Z}=\sum_{v=1}^{V} \boldsymbol{\alpha}_{v} \mathbf{S}^{(v)}
 \label{eq:direct add anchor graph}
\end{equation}
where $\mathbf{Z}$ represents the consensus graph obtained from all views, which contain both consensus information and redundant information from each view. $\boldsymbol{\alpha}_{v}$ and $\mathbf{S}^{(v)}$ represent the weight and anchor graph of the $v$-th view respectively.

Eq. (\ref{eq:direct add anchor graph}) can retain most of the information from all anchor graphs. However, Multi-view data provides different perspectives of the same samples, with each view containing both consensus and redundant information. For instance, an article might be reported in English, French, and Spanish. While these reports convey the same core information, such as news themes or sentiment, they also include redundant details specific to each language, such as variations in lexical and grammatical structures. Similarly, when a sample is represented by both text and images, the text view includes consensus information alongside redundant details unique to textual representation. Therefore, the consensus anchor graph $\mathbf{Z}$ obtained by Eq. (\ref{eq:direct add anchor graph}) contains a significant amount of redundant information. 

For clustering task, information consisting solely of principal components tends to achieve better clustering performance, as it ensures separability among samples. Therefore, due to the presence of redundant information, directly using Z for clustering analysis does not adequately ensure separability among samples, which hinders the improvement of clustering accuracy. To address this challenge, we aim for the consensus graph shared across all views to be low-rank. Therefore, we employ the nuclear norm of the consensus graph as a constraint in its learning process. By adaptively adjusting the weight of each view, we learn a low-rank consensus graph:
\begin{equation}
\label{eq:low-rank learing}
\begin{split}
 \min\limits_{\mathbf{Z},\alpha}&\left \|\mathbf{Z}- \sum_{v}^{V}  \boldsymbol{\alpha}_{v} \mathbf{S}^{(v)} \right \|_{F}^{2} + \beta \left \| \mathbf{Z} \right \| _{\ast } \\
 s.t. &\boldsymbol{1}^{\top} \boldsymbol{\alpha} =1,\boldsymbol{\alpha} \ge 0
\end{split}
\end{equation}
where $\beta$ is a trade-off parameter.

\subsection{Getting clustering indicator in the same step}
The majority of current anchor graph-based multi-view clustering methods learn the embedding representation of the anchor graph to approximate either the full-sample graph:
\begin{equation}
\label{eq:svd for anchor graph}
    svd\left(\mathbf{S}\right)=\mathbf{G}_{X}\mathbf{\Sigma}\mathbf{G}_{U}
\end{equation}
where $svd(\bullet)$ is the singular value decomposition operator. $\mathbf{G}_{X}$ is a new embedding representation for all samples.  These methods exhibit linear complexity in obtaining anchor graph embedding representations, significantly expanding the applicability of AGMC to large-scale data. However, these methods require additional post-processing steps after obtaining graph embedding, which leads to a loss of information and efficiency. To circumvent these post-processing steps, some methods integrate the learning of the clustering indicator and the learning of the embedding representation in a single step. Nevertheless, they still need to learn extra variables after getting graph embedding, which is inefficiency.

Through constructing anchor graphs and learning a consensus graph, we obtain a low-dimensional consensus anchor graph that contains common principal component information among views. Due to the above characteristics of consensus anchor graph, performing matrix decomposition on the consensus anchor graph is effective. Therefore, instead of applying SVD to the consensus anchor graph to obtain the graph embedding, we directly perform matrix factorization on the consensus anchor graph. To avoid post-processing steps, inspired by NMF \cite{lee1999learning}, We impose a non-negtive constraint on one of the factor matrices in the decomposition, denoted as $\mathbf{F}$, which serves as a soft indicator matrix:
\begin{equation}
     \begin{split}
 &\min\limits_{\mathbf{F},\mathbf{G}}  \left \| \mathbf{Z}-\mathbf{FG}^{\top} \right \|_{\rm{F}}^{2}
\\
 s.&t. \mathbf{G}^{\top}\mathbf{G}=\mathbf{I}, \mathbf{F}\ge 0 
\end{split}
\end{equation}
 where $\mathbf{I}$ is a identity matrix. For the theoretical analysis, $\mathbf{G} \in \mathbb{R}^{m\times c}$ represents a set of basis vectors with dimension $m$, in which $c$ is the number of clusters. $\mathbf{F} \in \mathbb{R}^{n\times c}$ serves as the soft indicator matrix, with $\mathbf{F}_{ij}$ representing the projection of the $i$-th sample onto the $j$-th basis vector. Finally, the condition $\mathbf{G}^{\top}\mathbf{G}=\mathbf{I}$ ensures the orthogonality of the basis vectors, facilitating better distinction of the original data. The physical meaning of matrix $\mathbf{F}$ is clear: $\mathbf{G}_{\cdot j}$ actually represents the $j$-th category. Since $\mathbf{F} \ge 0$, $\mathbf{F}_{i \cdot}$ actually represents the probability of the $i$-th sample on all $c$ categories. Therefore, we only need to take the largest element of $\mathbf{F}_{i \cdot}$ to get the category to which F belongs. Ultimately, we can get clustering indicator from the matrix $\mathbf{F}$. The objective function is formulated as:
\begin{equation}
\label{eq:objective function}
 \begin{split}
 \min\limits_{\mathbf{Z},\boldsymbol{\alpha},\mathbf{F},\mathbf{G}} &\left \| \mathbf{Z}  - \sum_{v}^{V} \boldsymbol{\alpha}_{v} \mathbf{S}^{(v)} \right \|_{\rm{F}}^{2} + \beta \left \| \mathbf{Z} \right \| _{\ast } + \gamma \left \| \mathbf{Z}-\mathbf{FG}^{\top} \right \|_{\rm{F}}^{2}
\\
 s.t. &\mathbf{G}^{\top}\mathbf{G}=\mathbf{I},\textbf{1}^{\top}\boldsymbol{\alpha} =1,\boldsymbol{\alpha}\ge 0, \mathbf{F}\ge 0 
\end{split}
\end{equation}
where $\gamma$ is a traded-off parameter similar to $\beta$. According to Eq. (\ref{eq:objective function}) , we can complete the learning and analysis of the consensus anchor graph in one step to directly obtain the clustering indicator. More importantly, under the nuclear norm constraint, Eq. (\ref{eq:objective function}) integrates the low-rank learning of the consensus anchor graph and the clustering indicator acquisition into a single unified framework. In this framework, the consensus graph learning is guided not only by the low-rank constraint but also by the clustering target. This joint optimization ensures that the learned consensus graph is both low-rank and well-aligned with the clustering structure, leading to enhanced clustering performance. This coupling of low-rank learning and clustering analysis distinguishes our method from existing approaches, which typically treat these steps independently.

\section{Optimization and Analysis}
In this section, we propose a fast optimization method for OMCAL. Specifically, we transform the problem of optimizing Eq. (\ref{eq:objective function}) into four steps. At each step, we optimize a single parameter while keeping the others fixed until the objective function converges.

\subsection{Initialization}
Before optimizing Eq. (\ref{eq:objective function}), we need to initialize three variables: $\boldsymbol{\alpha}$, $\mathbf{F}$, $\mathbf{G}$ and $\mathbf{Z}$. For the weight of each view, we initialize them evenly $\boldsymbol{\alpha}_{v} = 1/V$. For consensus graph $\mathbf{Z}$, we initialize it as $\mathbf{Z} = \sum_{v}^{V}\boldsymbol{\alpha}_{v}\mathbf{S}^{(v)}$. For other variables, we randomly initialize $\mathbf{F}$ to be a non-negative matrix $\mathbf{F} \ge 0$ and  randomly initialize $\mathbf{G}$ to be a orthogonal matrix $\mathbf{G}^{\top}\mathbf{G}=\mathbf{I}$.

\subsection{Solving $\mathbf{F}$}
When we fix other parameters $\mathbf{G}$, $\mathbf{Z}$ and $\boldsymbol{\alpha}$, and only update $\mathbf{F}$, we can rewritten Eq. (\ref{eq:objective function}) as follows:
\begin{equation}
\label{eq:transe obj for F}
 \min_{\mathbf{F}\ge 0}\left \| \mathbf{Z}-\mathbf{F}\mathbf{G}^{\top} \right \| _{\rm{F}}^{2}
\end{equation}

Solving Eq. (\ref{eq:transe obj for F}) is equivalent to finding $\mathbf{F}$ such that $\mathbf{Z}=\mathbf{FG}^{\top}$. Civen that $\mathbf{G}^{\top}\mathbf{G}=\mathbf{I}$, we can reformulate the problem $\mathbf{Z}=\mathbf{FG}^{\top}$ as $\mathbf{ZG} = \mathbf{F}$. Therefore, solving Eq. (\ref{eq:transe obj for F}) is equivalent to solving the following problems:
\begin{equation}
 \min_{\mathbf{F}\ge 0}\left \| \mathbf{ZG} - \mathbf{F} \right \| _{\rm{F}}^{2}
\end{equation}

Since the elements of $\mathbf{F}$ are non-negative, and the elements of $\mathbf{ZG}$ may contain negative values, any element $\mathbf{F}_{ij}$ in $\mathbf{F}$ can be solved as follows:
\begin{equation}
\label{eq:solve F}
 \mathbf{F}_{ij}=\max \left ( \left ( \mathbf{ZG} \right ) _{ij},0  \right ) 
\end{equation}

\subsection{Solving $\mathbf{G}$}
When we fix the other three parameters and update only $\mathbf{G}$, Eq. (\ref{eq:objective function}) can be rewritten as follows:
\begin{equation}
\label{eq:transe obj for G}
 \min_{\mathbf{G}^{\top}\mathbf{G}=\mathbf{I}}\left \| \mathbf{Z}-\mathbf{FG}^{\top} \right \| _{\rm{F}}^{2}
\end{equation}

According to the properties of Frobenius norm $\left \| \mathbf{A} \right \|_{\rm{F}} = \sqrt{\rm{Tr}(\mathbf{A}^{T}\mathbf{A}  )} $, we can obtain $\mathbf{G}$ through the resolution of the following problem:
\begin{equation}
 \min _{\mathbf{G}^{\top}\mathbf{G}=\mathbf{I}} \rm{Tr}\left ( \mathbf{ZZ}^{\top}-2\mathbf{G}^{\top}\mathbf{Z}^{\top}\mathbf{F}+\mathbf{F}^{\top}\mathbf{F} \right )
\end{equation}

Since other parameters are fixed except $\mathbf{G}$, we only solve the following problem:
\begin{equation}
 \label{eq:solve G trace}
 \max _{\mathbf{G}^{\top}\mathbf{G}=\mathbf{I}} {\rm Tr} \left ( \mathbf{G}^{\top}\mathbf{W} \right )
\end{equation}
where $\mathbf{W}=\mathbf{Z}^{\top}\mathbf{F}$.

\begin{theo}
\label{Theorem 1}
The solving of $\max_{\mathbf{M}^{\top}\mathbf{M}=\mathbf{I}} {\rm Tr} \left ( \mathbf{M}^{\top}\mathbf{N} \right ) $ is $\mathbf{M} = \mathbf{\check{A}} \left [\mathbf{I},\mathbf{0}\right ] \mathbf{\check{C}} $ ,where $\left [ \mathbf{\check{A}},\mathbf{\check{B}} , \mathbf{\check{C}}^{\top} \right ]=svd\left (  \mathbf{N}\right )$.
\end{theo}

\begin{proof}
If $\left [ \mathbf{\check{A}},\mathbf{\check{B}} , \mathbf{\check{C}}^{\top} \right ]=svd\left ( \mathbf{N} \right )$, then the following equation holds:
\begin{equation}
\begin{aligned} 
 {\rm Tr}\left( \mathbf{M}^{\top}\mathbf{N} \right )
 &= {\rm Tr} \left( \mathbf{M}^{\top}\mathbf{\check{A}}\mathbf{\check{B}}\mathbf{\check{C}}^{\top} \right ) \\
 &= {\rm Tr} \left( \mathbf{\check{B}}\mathbf{\check{C}}^{\top}\mathbf{M}^{\top}\mathbf{\check{A}} \right )\\
\end{aligned}
\label{eq:for_prove_decomposition}
\end{equation}

According to the properties of singular value decomposition, it is straightforward to prove that $\mathcal{A} = \mathbf{\check{C}}^{\top}\mathbf{M}^{\top}\mathbf{\check{A}}$ is an orthogonal matrix. Therefore, Eq. (\ref{eq:for_prove_decomposition}) can be expressed as:
\begin{equation}
 {\rm Tr}\left ( \mathbf{M}^{\top}\mathbf{N} \right ) = \sum_{ii}\mathbf{\check{B}}_{ii}\mathcal{A}_{ii}
\end{equation}

Since $\mathbf{\mathcal{A}}$ is an orthogonal matrix, $\left | \mathbf{\mathcal{A}}_{ii} \right | \le 1 $ holds. Additionally, Since $\mathbf{\check{B}}_{ii}$ is the singular value of $\mathbf{N}$, $\mathbf{\check{B}}_{ii} \ge 0$ holds. So the following inequality holds:
\begin{equation}
 {\rm Tr}\left ( \mathbf{M}^{\top}\mathbf{N} \right ) = \sum_{ii}\mathbf{\check{B}}_{ii}\mathcal{A}_{ii} \le \sum_{ii}\mathbf{\check{B}}_{ii}
\end{equation}

When $\mathbf{\mathcal{A}}_{ii}=1$, $\rm{Tr}\left ( \mathbf{M}^{\top}\mathbf{N} \right ) $ reach its maximum value. Therefore, we set $\mathbf{\mathcal{A}} = \left [\mathbf{I},\mathbf{0} \right ]$, Consequently, we obtain $\mathbf{\check{C}}^{\top}\mathbf{M}^{\top}\mathbf{\check{A}} = \left[ \mathbf{I},\mathbf{0} \right]$. Then, we have $\mathbf{M} = \mathbf{\check{A}} \left[ \mathbf{I}, \mathbf{0} \right] \mathbf{\check{C}}$.

\end{proof}

According to the \textbf{Theorem \ref{Theorem 1}}, we can solve Eq. (\ref{eq:solve G trace}) as follows:
\begin{equation}
 \label{eq:solve G}
 \mathbf{G} = \mathbf{\check{\Delta }} \left [ \mathbf{I},\mathbf{0} \right ] \mathbf{\check{\Lambda }}
\end{equation}
where $ \left [\mathbf{\check{\Delta }},\mathbf{\check{\Theta }} ,\mathbf{\check{\Lambda }} \right ] =svd\left ( \mathbf{Z}^{\top}\mathbf{F} \right )$.

\subsection{Solving $\mathbf{Z}$}
When we fix the other three parameters except for $\mathbf{Z}$, Eq. (\ref{eq:objective function}) can be reformulated as follows:
\begin{equation}
\label{eq:solve Z trace }
\begin{split}
 \min_{\mathbf{Z}}{\rm Tr}&\left [ \left(1+\gamma \right)\mathbf{ZZ}^{\top}-
 2\mathbf{Z}^{\top}\left ( \sum_{v}\boldsymbol{\alpha}_{v}\mathbf{S}^{(v)}+\gamma \mathbf{FG}^{\top} \right ) \right ]\\+&\beta \left \| \mathbf{Z} \right \|_{\ast } 
\end{split}
\end{equation}

Since we only update $\mathbf{Z}$ in this step, adding a constant term $\mathbf{M}$ does not affect the optimal solution of $\mathbf{Z}$. In accordance with the characteristics of the Frobenius norm, solving Eq. (\ref{eq:solve Z trace }) is equivalent to solving the following problem:
\begin{equation}
 \label{eq:for IT}
 \min_{\mathbf{Z}}\frac{1}{2}\left \| \mathbf{Z}-\mathbf{M} \right \|_{\rm{F}}^{2} + \frac{\beta }{2(1+\gamma)}\left \| \mathbf{Z} \right \|_{\ast } 
\end{equation}
where $\mathbf{M} = \frac{ {\textstyle \sum_{v} \boldsymbol{\alpha}_{v}\mathbf{S}_{v}+\gamma \mathbf{FG}^{\top}}}{1+\gamma } $.

Based on the the singular value thresholding algorithm \cite{cai2010singular}, the solution to Eq. (\ref{eq:for IT}) can be obtained as follows:
\begin{equation}
\label{eq:solve Z}
 \mathbf{Z}=\mathbf{P}\mathcal{D}_{\frac{\beta}{2(1+\gamma)}}\left [ \boldsymbol{\Sigma } \right ] \mathbf{Q}^{\top}
\end{equation}
where $\mathbf{P} \boldsymbol{\Sigma} \mathbf{Q}^{\top} = svd(\mathbf{M})$, $\mathcal{D}_{\frac{\beta}{2(1+\gamma)}}\left [ \boldsymbol{\Sigma } \right ]$ represents the application of a soft-thresholding rule to each element $\boldsymbol{\Sigma }_{ii}$ in $\boldsymbol{\Sigma }$. Specifically, $\boldsymbol{\Sigma }_{ii} =\left( \boldsymbol{\Sigma }_{ii}-\frac{\beta}{2(1+\gamma)} \right )_{+} $, where $t_{+}$ denotes the positive part of $t$, defined as $t_{+} = \max \left ( t,0 \right ) $.

\subsection{Solving $\boldsymbol{\alpha}$}
Fixing the other three parameters except for $\boldsymbol{\alpha}$, we can rewrite Eq. (\ref{eq:objective function}) as follows:
\begin{equation}
\label{eq:trans for alpha}
\begin{split}
 \min_{\boldsymbol{\alpha} }\left \| \mathbf{Z}-\sum_{v}\boldsymbol{\alpha}_{v} \mathbf{S}^{(v)} \right \|_{\rm{F}}^{2} \\
 s.t. \boldsymbol{1}^{\top}\boldsymbol{\alpha}=1,\boldsymbol{\alpha} \ge 0
\end{split}
\end{equation}

Referring to the characteristics of the Frobenius norm, We can express Frobenius norm as a matrix trace. After removing the variables that are irrelevant to $\boldsymbol{\alpha}$, Solving Eq. (\ref{eq:trans for alpha}) is equivalent to solving the following problem:
\begin{equation}
\label{eq:alpha trace}
\begin{split}
  \min_{\boldsymbol{\alpha}} {\rm Tr}& \left [ - 2\sum_{v} \left ( \boldsymbol{\alpha}_{v} \mathbf{S}^{(v)} \right ) ^{\top } \mathbf{Z}+\left ( \sum_{v} \boldsymbol{\alpha}_{v} \mathbf{S}^{(v)} \right )^{\top} \sum_{v} \boldsymbol{\alpha}_{v} \mathbf{S}^{(v)} \right ] \\
& s.t. \boldsymbol{1}^{\top}\boldsymbol{\alpha}=1,\boldsymbol{\alpha} \ge 0
\end{split}
\end{equation}

The trace of a matrix represents the sum of its diagonal elements of the matrix. If $\bar{d} $ is the number of rows in matrix $\mathbf{A}$ and the number of columns in matrix $\mathbf{B}$, we can get ${\rm Tr} \left ( \mathbf{AB} \right ) = \sum_{i=1}^{\bar{d} } \mathbf{A}_{ i\cdot}\mathbf{B}_{\cdot j}$. Furthermore, Eq. (\ref{eq:alpha trace}) can be rewritten as follows:
\begin{equation}
 \min_{\boldsymbol{1}^{\top}\boldsymbol{\alpha} =1,\boldsymbol{\alpha} \ge 0 } \boldsymbol{\alpha} ^{\top} \mathbf{\widetilde{S} \widetilde{S}^{\top}} \boldsymbol{\alpha} -2vec(\mathbf{Z})\mathbf{\widetilde{S}}^{\top}\boldsymbol{\alpha} 
\end{equation} 
where $vec\left ( \cdot\right ) $ represents the concatenation of vectors from the matrix into a new vector, $\mathbf{\widetilde{S}} = \left [ \mathbf{\widetilde{S}}_{11}^{\top},\dots,\mathbf{\widetilde{S}}_{nm}^{\top} \right ]^{\top} $, and $\mathbf{\widetilde{S}}_{ij} = \left [ \mathbf{S}_{ij}^{(1)},\dots,\mathbf{S}_{ij}^{(V)} \right ]^{\top} $. Let $\mathbb{M}=\mathbf{\widetilde{S}\widetilde{S}}^{\top}$ and $\boldsymbol{m} =2\mathbf{\widetilde{S}}vec(\mathbf{Z})^{\top}$, the problem of solving $\boldsymbol{\alpha}$ is equal to the following problem:
\begin{equation}
\label{eq:model for quadratic optimization }
 \min_{\boldsymbol{1}^{\top}\boldsymbol{\alpha}=1,\boldsymbol{\alpha} \ge 0} \boldsymbol{\alpha} ^{\top}\mathbb{M}\boldsymbol{\alpha} -\boldsymbol{\alpha}^{\top} \boldsymbol{m}
\end{equation}

The above problem is a quadratic programming problem, which can be solved by augmented lagrangian method \cite{nie2019multiview}.

Finally, the clustering indicator can be directly obtained from the indicator matrix $\mathbf{F}$. Algorithm \ref{algor for OMCAL} summarizes the optimization process of Eq. (\ref{eq:objective function}).
\begin{algorithm}
\caption{\textbf{Algorithm} of OMCAL}
\begin{algorithmic}[1]
 \REQUIRE Original anchor graphs $\mathcal{S}=\left[\mathbf{S}^{(1)},\mathbf{S}^{(2)},\dots,\mathbf{S}^{(V)} \right]$, the number of clusters $c$, trade-off parameters $\beta$ and $\gamma$
 \ENSURE The indicator matrix $\mathbf{F}$
 \STATE \textbf{Initialize} $\boldsymbol{\alpha}_{v}=1/V$, $\mathbf{Z} = \sum_{v}^{V}\boldsymbol{\alpha}_{v}\mathbf{S}^{(v)}$,  the orthogonal matrix $\mathbf{G}$ and the non-negative matrix $\mathbf{F}$
 \WHILE{not converge}
 \STATE Update $\mathbf{F}$ according to Eq. (\ref{eq:solve F})
 \STATE Update $\mathbf{G}$ according to Eq. (\ref{eq:solve G})
 \STATE Update $\mathbf{Z}$ according to Eq. (\ref{eq:solve Z})
 \STATE Update $\boldsymbol{\alpha}$ according to Eq. (\ref{eq:model for quadratic optimization })
 \ENDWHILE
\end{algorithmic}
\label{algor for OMCAL}
\end{algorithm}

\subsection{Computational complexity analysis}
The computational complexity of our proposed method mainly consists of two parts: constructing the anchor graphs and the optimization processing. The process of constructing anchor graphs includes obtaining anchors and computing similarity among anchors and samples, with the computational complexity of $\mathcal{O}\left( t_{1}nmd \right)$ and $\mathcal{O}\left( nmd \right)$ respectively. Here, $t_{1}$ is the number of iterations of k-means, $d = \sum _{i=1}^{V} d_{i}$, and $d_{i}$ is the feature dimension of the $i$-th view. The optimization process consist of four steps: solving $\mathbf{F}$, solving $\mathbf{G}$, solving $\mathbf{Z}$ and solving $\boldsymbol{\alpha}$, with computational complexities of $\mathcal{O} \left( nc \right)$, $\mathcal{O} \left( mc^{2} \right)$, $\mathcal{O} \left( m^{2}n \right)$ and $\mathcal{O} \left( V \right)$, respectively. Therefore, the overall computational complexity of the optimization process is $\mathcal{O} \left[ \left( nc + mc^{2} + m^{2}n + V \right) \times t_{2} \right]$, where $t_{2}$ is the number of iterations of optimization. Based on the above analysis, the computational complexity of our proposed method is $\mathcal{O}\left [ t_{1}nmd + nmd + t_{2} \times \left ( nc + mc^{2} + m^{2}n + V \right ) \right ]$. When faced with large-scale datasets, $m$, $c$, $d$, $t_{1}$, and $t_{2}$ are much smaller than $n$. Therefore, our method can be considered linear, which is efficient when dealing with large-scale data.

\section{Experiments}
In this section, we present the comparison between our methods and state-of-the-art multi-view clustering methods. Specifically, we ran all methods on seven real-world datasets including normal and large-scale datasets, and utilized six metrics to evaluate the performance of different clustering methods. These metrics include clustering accuracy (ACC), normalized mutual information (NMI), cluster purity (Purity), adjusted rand index (ARI), F-score, and Precision \cite{zhan2018multiview}. For each of the six metrics, elevated values suggest superior clustering performance. Furthermore, we conducted experiments with various parameter settings to analyze the sensitivity of our method to these parameters. Finally, we visualized the consensus anchor, and conducted extensive experiments to evaluate the effectiveness of the proposed consensus anchor graph learning strategy. All experiments were conducted on an Intel(R) Xeon(R) Gold 5318H CPU @ 2.50GHz 128GB workstation, using Matlab R2022a.

\subsection{Datasets and baselines}
We conducted an extensive evaluation of our approach on seven real-world datasets, including Coil \cite{zhan2017graph}, WiKi \cite{rasiwasia2010new}, USPS \cite{lv2021pseudo}, Reuters \cite{apte1994automated}, XMedia \cite{peng2017overview}, NoisyMNIST \cite{sun2024robust}, Cifar10 \cite{krizhevsky2009learning}, Cifar100 \cite{krizhevsky2009learning}, and Mnist \cite{lecun1998gradient}. The last six datasets are large-scale, each containing more than 10,000 samples. Table \ref{Tab:detail of datasets} provides a summary of these datasets.

\begin{table}[htbp]
\begin{adjustbox}{center}
\begin{threeparttable}
\caption{\label{Tab:detail of datasets}The Detail Of Different Multi-view Datasets}
 \centering
 \makeatletter 
 \begin{tabular*}{8.5cm}{@{\extracolsep{\fill}}lccc}
 \toprule[1pt]
 Dataset &  Samples& Clusters & $d^{(1)}$, $d^{(2)}$, $\cdots$, $d^{(V)}$                    \\
 \midrule[0.2pt]
Coil                      & 1440                                   & 20                     & 30,19,30                \\
Wiki                      & 2866                                 & 10                     & 128, 10                       \\
USPS                        & 9298                                  & 10                    & 30, 9, 30                       \\
Reuters                     & 18758                                 & 6                     & 21531, 24892, 34251, 15506, 11547 \\
NoisyMNIST                     & 30000                                 & 10                     & 784,784 \\
XMedia                  & 40000                                 & 200                     & 4096,300 \\
Cifar10                       & 50000                                & 10                    & 512, 2048, 1024                      \\
Cifar100                        & 50000                                 & 100                    & 512, 2048, 1024             \\
Mnist                   & 60000                                 & 10                    & 342, 1024, 64       \\
\bottomrule[1pt]
 \end{tabular*}
\end{threeparttable}
\end{adjustbox}
\end{table}

To demonstrate the performance of our proposed method, we evaluate it alongside several state-of-the-art multi-view clustering methods, which include AMGL \cite{nie2016parameter}, MLAN \cite{nie2017multi}, MCGC \cite{zhan2018multiview}, SMC \cite{liu2022scalable}, FSMSC \cite{chen2023fast}, SLMVGC \cite{tan2023sample}, UOMvSC \cite{tang2022unified}, MVSC-HFD \cite{ou2024anchor}. In all experiments, we utilized the open-source code or original code of the aforementioned methods for parameter tuning. To the best of our ability, we select the best performance of each method.

Table \ref{Tab:computational complexity compare} summarizes the computational complexity of all methods, where $t_{i}$ represents the number of iterations, $n$ represents the number of samples, $d = \sum _{i=1}^{V} d_{i}$ represents the sum of the dimensions of all views, $c$ represents the number of clusters, $V$ represents the nunber of views, $m$ represensts the number of anchors, $l_{i}$ represents the dimension of transfer matrix in MVSC-HFD. As shown in Table \ref{Tab:computational complexity compare}, AMGL, MLAN, MCGC, SLMVGC and UOMvSC are clustering methods based on the full-sample graphs. Their graph construction and analysis steps require computational complexities of $\mathcal{O}\left( n^{2}d \right)$ and $\mathcal{O}\left( n^{2}c \right)$ respectively, resulting in quadratic computational complexity. On the other hand, although SMC is a clustering method based on the anchor graph, it utilizes the full-sample graph to perform graph filtering operations on the original data. Consequently, constructing the full-sample graph still requires $\mathcal{O}\left( n^{2}d \right)$ computational complexity, leading to its quadratic computational complexity as well. In contrast, FSMSC, MVSC-HFD, and OMCAL do not introduce the full-sample graph in the clustering process, but instead construct and analyze the anchor graph, resulting in linear computational complexity. Furthermore, although FSMSC and MVSC-HFD also exhibit linear computational complexity, their actual complexity is scaled by different coefficients due to the presence of additional variables to be optimized. This makes the theoretical computational complexity of our method more favorable.

\begin{table}[ht]
\begin{adjustbox}{center}
\begin{threeparttable}
\caption{\label{Tab:computational complexity compare}The Computational Complexity of All Methods}
 \centering
 \makeatletter 
 \begin{tabular*}{8.5cm}{@{\extracolsep{\fill}}cc}
 \toprule[1pt]
 Method & Time complexity \\
 \midrule[0.2pt]
 AMGL & $\mathcal{O}\left ( n^{2}d+t_{1}n^{2}c+t_{1}n^{2}cV+t_{2}nc^{2} \right )$   \\
 MLAN & $\mathcal{O}\left ( n^{2}d+t_{1}n^{2}d+t_{1}n^{2}+t_{1}n^{2}c \right )$   \\
 MCGC & $\mathcal{O}\left ( n^{2}d+t_{1}n^{2}cV+t_{1}t_{2}n^{2}c \right )  $  \\
 SMC & $\mathcal{O}\left ( n^{2}d+t_{1}n^{2}d+t_{2}nmd+nmd+t_{3}nm^{2} \right )  $ \\
 FSMSC & $\mathcal{O}\left ( t_{1}ncd+t_{1}nc^{2}+t_{1}nc^{3}+t_{2}nc^{2} \right )  $ \\
 SLMVGC & $\mathcal{O}\left ( n^{2}d+t_{1}nV+t_{1}n^{2}+t_1{nm^{2}+t_{1}n^{3}} +t_{2}n^{2}c \right )  $  \\
 UOMvSC & $\mathcal{O}\left ( n^{2}d+n^{2}cV+t_{1}V+t_{1}nc^{2}+t_{1}nc \right )  $ \\
 \multirow{2}{*}{MVSC-HFD} & $\mathcal{O}\left ( t_{1}Vdlc+t_{1}Vc\sum_{i=0}^{\delta-2}l_{i}l_{i+1}  \right. $ \\ & $ \left. +t_{1}ndc+t_{1}mdc+t_{1}nmc+t_{2}nmc \right )  $  \\
 OMCAL & $\mathcal{O} \left ( t_{1}nmd+t_{2}nc+t_{2}mc^{2}+t_{2}nm^{2}+t_{2}V \right )$ \\
\bottomrule[1pt]
 \end{tabular*}
\end{threeparttable}
\end{adjustbox}
\end{table}

\subsection{Clustering results}

 In this subsection, we present the outcomes of eight clustering methods across seven datasets. As outlined earlier, we ensure fairness by tuning parameters for each method on all datasets. All methods were executed five times with the best parameter settings. Table \ref{Tab:detail of parameters} shows the parameter settings adopted by our algorithm on various datasets. Regarding efficiency, Table \ref{Tab:running time} displays the average time taken by all methods on all datasets. For effectiveness, Table \ref{Tab:result} presents the average value and standard deviation for each method on all datasets. Due to the large scale and high dimensions of certain datasets, some methods encountered memory overflow or extended running times. Specifically, we record `OM' when a method experiences memory overflow and `---' when the running time of a method exceeds 6 hours.

 \begin{table}[ht]
\begin{adjustbox}{center}
\begin{threeparttable}
\caption{\label{Tab:detail of parameters}The Detail Of OMCAL's Parameters for all Datasets}
 \centering
 \begin{tabular*}{7.5cm}{@{\extracolsep{\fill}}lccc}
 \toprule[1pt]
 Dataset &  $m$ &  $\beta$ & $\gamma$  \\
 \midrule[0.2pt]
Coil   & 35 & 0.3 & 0.01 \\
Wiki & 30 & 0.1 & 0.1 \\
USPS & 40 & 0.3 & 0.1 \\
Reuters & 15 & 0.8 & 0.0001 \\
NoisyMNIST & 100 & 0.2 & 1 \\
XMedia & 40 & 0.1 & 1 \\
Cifar10 & 35 & 1 & 0.001 \\
Cifar100 & 150 & 0.4 & 0.1 \\
MNIST & 30 & 0.2 & 0.1 \\
\bottomrule[1pt]
 \end{tabular*}
\end{threeparttable}
\end{adjustbox}
\end{table}

\begin{table*}[!h]
 \begin{adjustbox}{center}
\begin{threeparttable}
\caption{\label{Tab:result}Clustering Results Comparison on All Datasets}
 \centering
 \setlength{\tabcolsep}{1.2mm}
 \makeatletter
 \begin{tabular*}{19cm}{llccccccccc}
 \toprule[1pt]
Metrics & Methods & Coil & Wiki & USPS & Reuters & NoisyMNIST & XMedia  & Cifar10 & Cifar100 & Mnist \\
\midrule[0.2pt]
\multirow{9}{*}{ACC} & AMGL & 0.625(0.000) & 0.159(0.000) & 0.706(0.000) & --- & ---  & ---   & OM & OM & OM \\
 & MLAN & \textbf{1.000(0.000)} & 0.158(0.010) & 0.168(0.002) & 0.323(0.000) & 0.597 (0.000)  & OM   & OM & OM & OM \\ 
 & MCGC & \textbf{1.000(0.000)} & 0.366(0.000) & 0.779(0.000) & 0.335(0.000) & ---  & OM   & OM & OM & OM  \\
 & SMC & 0.738(0.069) & 0.171(0.000) & 0.622(0.010) & 0.432(0.000) & 0.311 (0.016)  & 0.764 (0.030)   & OM & OM & OM \\
 & FSMSC & 0.863(0.000) & 0.209(0.000) & 0.695(0.000) & \underline{0.453(0.000)} & 0.474 (0.000)  & \underline{0.867 (0.000)}  & \underline{0.854(0.000)} & \underline{0.859(0.000)} & 0.710(0.000) \\
 & SL\_MVGC & 0.338(0.021) & 0.552(0.000) & 0.572(0.000) & 0.269(0.000) & OM  & OM   & OM & OM & OM \\
 & UOMvSC &\textbf{1.000(0.000}) & \underline{0.581(0.000)} & 0.670(0.000) & OM & \underline{0.640 (0.000)}  & OM   & OM & OM & OM \\
 & MVSC-HFD & 0.813(0.000) & 0.564(0.000) & \underline{0.811(0.000)} & OM & 0.507 (0.000)  & 0.536 (0.000)   & 0.880(0.000) & 0.626(0.000) & \underline{0.885(0.000)} \\
 & \textbf{OMCAL} & \textbf{1.000(0.000)} & \textbf{0.594(0.015)} & \textbf{0.882(0.005)} & \textbf{0.463(0.051)} & \textbf{0.737 (0.021)}  & \textbf{0.888 (0.030)}   & \textbf{0.970(0.004)} & \textbf{1.000(0.000)} & \textbf{0.989(0.001)} \\
\midrule[0.2pt]
\multirow{9}{*}{NMI} & AMGL & 0.674(0.000) & 0.022(0.000) & 0.694(0.000) & --- & ---  & ---   & OM & OM & OM \\
 & MLAN & \textbf{1.000(0.000)} & 0.015(0.003) & 0.004(0.002) & 0.060(0.000) & 0.672 (0.000)  & OM   & OM & OM & OM \\ 
 & MCGC & \textbf{1.000(0.000)} & 0.344(0.000) & 0.753(0.000) & 0.084(0.000) & ---  & OM   & OM & OM & OM \\
 & SMC & 0.882(0.035) & 0.040(0.000) & 0.656(0.006) & 0.213(0.000) & 0.280 (0.006)  & 0.871 (0.008)   & OM & OM & OM \\
 & FSMSC & 0.954(0.000) & 0.069(0.000) & 0.638(0.000) & \underline{0.265(0.000)} & 0.440 (0.000)  & \textbf{0.939 (0.000)}   & \underline{0.903(0.000} & \underline{0.970(0.000)} & \underline{0.849(0.000)} \\
 & SL\_MVGC & 0.534(0.018) & 0.485(0.000) & 0.691(0.000) & 0.005(0.000) & OM  & OM   & OM & OM & OM \\
 & UOMvSC & \textbf{1.000(0.000)} & \underline{0.548(0.000)} & \textbf{0.776(0.000)} & OM & \textbf{0.770 (0.000)}  & OM   & OM & OM & OM \\
 & MVSC-HFD & 0.931(0.000) & 0.535(0.000) & 0.727(0.000) & OM & 0.364 (0.000)  & 0.713 (0.000)   & 0.891(0.000) & 0.822(0.000) & \underline{0.888(0.000)} \\
 & \textbf{OMCAL} & \textbf{1.000(0.000)} & \textbf{0.562(0.004)} & \underline{0.767(0.007)} & \textbf{0.277(0.055)} & \underline{0.687 (0.025)}  & \underline{0.938 (0.007)}   & \textbf{0.927(0.006) } & \textbf{1.000(0.000)} & \textbf{0.968(0.003)} \\
\midrule[0.2pt]
\multirow{9}{*}{Purity} & AMGL & 0.630(0.000) & 0.176(0.000) & 0.708(0.000) & --- & ---  & ---   & OM & OM & OM \\
 & MLAN & \textbf{1.000(0.000)} & 0.171(0.000) & 0.170(0.002) & 0.324(0.000) & 0.605 (0.000)  & OM   & OM & OM & OM \\ 
 & MCGC  & \textbf{1.000(0.000)} & 0.450(0.000) & 0.779(0.000) & 0.343(0.000) & ---  & OM   & OM & OM & OM \\
 & SMC & 0.918(0.024) & 0.243(0.000) & 0.741(0.015) & \textbf{0.637(0.000)} & 0.412 (0.005)  & 0.840 (0.014)   & OM & OM & OM \\
 & FSMSC & 0.900(0.000) & 0.237(0.000) & 0.758(0.000) & 0.516(0.000) & 0.520 (0.000)  & \underline{0.895 (0.000) }  & \underline{0.888(0.000)} & \underline{0.900(0.000)} & 0.796(0.000) \\ 
 & SL\_MVGC & 0.483(0.014) & 0.616(0.000) & 0.714(0.000) & 0.273(0.000) & OM  & OM   & OM & OM & OM \\
 & UOMvSC & \textbf{1.000(0.000)} & \underline{0.633(0.000)} & 0.790(0.000) & OM & \underline{0.707 (0.000)}  & OM   & OM & OM & OM \\
 & MVSC-HFD & 0.850(0.000) & 0.609(0.000) & \underline{0.818(0.000)} & OM & 0.510 (0.000)  & 0.539 (0.000)   & 0.880(0.000) & 0.635(0.000) & \underline{0.886(0.000)} \\
 & \textbf{OMCAL} & \textbf{1.000(0.000)} & \textbf{0.633(0.004)} & \textbf{0.882(0.005)} & \underline{0.546(0.041)} & \textbf{0.780 (0.021) } & \textbf{0.898 (0.023)}   & \textbf{0.970(0.004)} & \textbf{1.000(0.000)} & \textbf{0.989(0.001)} \\
\midrule[0.2pt]
\multirow{9}{*}{ARI} & AMGL & 0.549(0.000) & 0.001(0.000) & 0.663(0.000) & --- & ---  & ---   & OM & OM & OM \\
 & MLAN & \textbf{1.000(0.000)} & -0.002(0.002) & -0.000(0.000) & 0.001(0.000) & 0.540 (0.000)  & OM   & OM & OM & OM \\ 
 & MCGC & \textbf{1.000(0.000)} & 0.185(0.000) & 0.746(0.000) & 0.007(0.000) & ---  & OM   & OM & OM & OM \\
 & SMC & 0.640(0.109) & 0.017(0.000) & 0.515(0.014) & 0.142(0.000) & 0.145 (0.003)  & 0.684 (0.027)   & OM & OM & OM \\ 
 & FSMSC & 0.877(0.000) & 0.037(0.000) & 0.574(0.000) & \textbf{0.217(0.000)} & 0.322 (0.000)  & \underline{0.846 (0.000)}   & 0.847(0.000) & \underline{0.884(0.000)} & 0.743(0.000) \\ 
 & SL\_MVGC & 0.143(0.038) & 0.395(0.000) & 0.528(0.000) & 0.000(0.000) & OM  & OM   & OM & OM & OM \\
 & UOMvSC & \textbf{1.000(0.000)} & \underline{0.455(0.000)} & 0.646(0.000) & OM & \underline{0.598 (0.000)}  & OM   & OM & OM & OM \\
 & MVSC-HFD & 0.834(0.000) & 0.434(0.000) & \underline{0.733(0.000)} & OM & 0.280 (0.000)  & 0.373 (0.000)   & \underline{0.856(0.000)} & 0.361(0.000) & \underline{0.850(0.000)} \\
 & \textbf{OMCAL} & \textbf{1.000(0.000)} & \textbf{0.478(0.013)} & \textbf{0.789(0.009)} & \underline{0.201(0.048)} & \textbf{0.632 (0.034)}  & \textbf{0.857 (0.024)}   & \textbf{0.936(0.007)} & \textbf{0.999(0.000)} & \textbf{0.977(0.003)} \\
\midrule[0.2pt]
\multirow{9}{*}{F-score} & AMGL & 0.577(0.000) & 0.160(0.000) & 0.706(0.000) & --- & ---  & ---   & OM & OM & OM \\ 
 & MLAN & \textbf{1.000(0.000)} & 0.186(0.007) & 0.194(0.000) & 0.344(0.000) & 0.603 (0.000)  & OM   & OM & OM & OM \\ 
 & MCGC & \textbf{1.000(0.000)} & 0.305(0.000) & 0.777(0.000) & 0.346(0.000) & ---  & OM   & OM & OM & OM \\
 & SMC & 0.663(0.101) & 0.128(0.000) & 0.569(0.013) & 0.375(0.000) & 0.246 (0.003)  & 0.686 (0.027)   & OM & OM & OM \\
 & FSMSC & 0.884(0.000) & 0.137(0.000) & 0.619(0.000) & \textbf{0.393(0.000)} & 0.392 (0.000)  & \underline{0.847 (0.000)}   & 0.863(0.000) & \underline{0.885(0.000)} & 0.772(0.000) \\ 
 & SL\_MVGC & 0.216(0.031) & 0.460(0.000) & 0.583(0.000) & 0.352(0.000) & OM  & OM   & OM & OM & OM \\
 & UOMvSC & \textbf{1.000(0.000)} & \underline{0.516(0.000)} & 0.686(0.000) & OM & \underline{0.646 (0.000)}  & OM   & OM & OM & OM \\
 & MVSC-HFD & 0.843(0.000) & 0.496(0.000) & \underline{0.762(0.000)} & OM & 0.366 (0.000)  & 0.377 (0.000)   & \underline{0.872(0.000)} & 0.372(0.000) & \underline{0.867(0.000)} \\
 & \textbf{OMCAL} & \textbf{1.000(0.000)} & \textbf{0.535(0.012)} & \textbf{0.811(0.008)} & \underline{0.378(0.041)} & \textbf{0.670 (0.031)}  & \textbf{0.857 (0.024) }  & \textbf{0.943(0.006)} & \textbf{0.999(0.000)} & \textbf{0.979(0.003)} \\ 
\midrule[0.2pt]
\multirow{9}{*}{Precision} & AMGL & 0.472(0.000) & 0.108(0.000) & 0.596(0.000) & --- & ---  & ---   & OM & OM & OM \\
 & MLAN & \textbf{1.000(0.000)} & 0.107(0.001) & 0.107(0.000) & 0.214(0.000) & 0.442 (0.000)  & OM   & OM & OM & OM \\ 
 & MCGC & \textbf{1.000(0.000)} & 0.223(0.000) & 0.684(0.000) & 0.217(0.000) & ---  & OM   & OM & OM & OM \\
 & SMC & 0.910(0.023) & 0.133(0.000) & 0.595(0.027) & \textbf{0.513(0.000)} & 0.297 (0.005)  & 0.756 (0.015)   & OM & OM & OM \\
 & FSMSC & 0.824(0.000) & 0.142(0.000) & 0.635(0.000) & 0.375(0.000) & 0.383 (0.000)  & \underline{0.806 (0.000)}   & \underline{0.806(0.000)} & \underline{0.826(0.000)} & 0.688(0.000) \\
 & SL\_MVGC & 0.127(0.026) & 0.462(0.000) & 0.530(0.000) & 0.214(0.000) & OM  & OM   & OM & OM & OM \\
 & UOMvSC & \textbf{1.000(0.000)} & \underline{0.501(0.000)} & 0.659(0.000) & OM & \underline{0.539 (0.000)}  & OM   & OM & OM & OM \\
 & MVSC-HFD & 0.757(0.000) & 0.483(0.000) & \underline{0.741(0.000)} & OM & 0.307 (0.000)  & 0.271 (0.000)   & 0.798(0.000) & 0.230(0.000) & \underline{0.788(0.000)} \\ 
 & \textbf{OMCAL} & \textbf{1.000(0.000)} & \textbf{0.531(0.014)} & \textbf{0.809(0.009)} & \underline{0.365(0.032)} & \textbf{0.651 (0.031)}  & \textbf{0.827 (0.034)}   & \textbf{0.942(0.006)} & \textbf{0.999(0.000)} & \textbf{0.979(0.003)} \\
\bottomrule[1pt]
\end{tabular*}
\begin{tablenotes}
\footnotesize
\item *The best results are shown in bold, while the second-best results are indicated with underlining
 \end{tablenotes}
\end{threeparttable}
 \end{adjustbox}
\end{table*}

 In terms of clustering effectivness, as shown in Table \ref{Tab:result},  OMCAL, MLAN, MCGC, and UOMvSC can reach 1 on all metrics for the Coil dataset. Different from the other three methods based on full-sample graphs, OMCAL is the sole anchor graph-based method, indicating that our proposed anchor graph learning strategy effectively captures structural information akin to that of the full-sample graph. In the remaining datasets, most of the methods proposed in recent years have better clustering performance. Furthermore, OMCAL can achieve first place in most metrics. The high clustering effectiveness of OMCAL can be attributed to several key factors: 1) The adaptive consensus anchor graph learning strategy based on kernel norms well eliminates redundant information and noise within the consensus anchor graph. 2) The integration of anchor graph learning and clustering indicator acquisition within the same framework extracts a clearer underlying clustering structure and prevents additional information loss.

\begin{table*}[htbp]
 \begin{adjustbox}{center}
 \begin{threeparttable}
 \caption{\label{Tab:running time} Comparison of running times across all datasets (seconds)}
 \centering
 \begin{tabular*}{17cm}{@{\extracolsep{\fill}}lccccccccc}
 \toprule[1pt]
 Datasets & AMGL & MLAN & MCGC & SMC & FSMSC & SLMVGC & UOMvSC & MVSC-HFD & \textbf{OMCAL} \\
 \midrule[0.2pt]
 Coil & 14.606 & \underline{1.124} & 20.894 & 1.415 & 23.721 & 6.535 & 4.721 & 1.780 & \textbf{1.098} \\
 Wiki & 27.631 & 20.140 & 96.680 & 3.676 & 44.598 & 64.153 & 25.743 & \underline{2.793} & \textbf{1.669} \\
 USPS & 2276.147 & 266.290 & 3002.102 & 63.615 & 139.847 & 803.251 & 469.928 & \underline{12.105} & \textbf{9.653} \\
 Reuters & --- & 5766.485 & 14514.476 & \underline{847.763} & 2193.054 & 5096.926 & OM & OM & \textbf{310.035} \\
NoisyMNIST & --- & 11180.000 & --- & 559.942 & 491.410 & OM & 6086.414 & \underline{103.92}& \textbf{97.977}\\
XMedia & --- & OM & OM & 1446.273 & 6559.800 & OM & OM & \underline{1176.900} & \textbf{541.994} \\
 Cifar10 & OM & OM & OM & OM & 925.584 & OM & OM & \underline{417.552} & \textbf{59.872} \\
 Cifar100 & OM & OM & OM & OM & 1919.104 & OM & OM & \underline{624.361} & \textbf{368.033} \\
 Mnist & OM & OM & OM & OM & 930.147 & OM & OM & \underline{206.343} & \textbf{136.638} \\
 \bottomrule[1pt]
 \end{tabular*}
 \begin{tablenotes}
 \footnotesize
 \item *The best results are shown in bold, while the second-best results are indicated with underlining
 \end{tablenotes}
 \end{threeparttable} 
 \end{adjustbox}
\end{table*}

In terms of clustering efficiency, as shown in Table \ref{Tab:running time}, OMCAL significantly advantage over state-of-the-art methods based on full-sample graphs, realizing efficiency gains extending from multiple times to several hundred times. For instance, on large-scale datasets Reuters, NoisyMNIST, XMedia, Cifar10, Cifar100, and MNIST, our method yields clustering indicator within 10 minutes, while other full-sample graph-based methods either exceed memory limits or have unacceptably long running times. Moreover, compared to advanced anchor graph-based methods, our approach still demonstrates superior efficiency, achieving the fastest computation speed across all datasets. The efficiency of our algorithm is attributable to two key factors: 1) OMCAL is an anchor graph-based method that measures only the similarity among anchors and samples, thereby mitigating the high computational complexity associated with constructing full-sample graphs; 2) OMCAL can obtain clustering indicator directly, eliminating the requirement for extra post-processing steps. Additionally, the optimization process of OMCAL does not introduce extra variables, further enhancing the computational efficiency.

 In addition to effectiveness and efficiency, we also present the convergence curve of OMCAL in Fig. \ref{fig:Convergencecurve}. The horizontal axis of Fig. \ref{fig:Convergencecurve} is the number of iterations, while the vertical axis indicates the objective function value. The figure illustrates that our proposed optimization method exhibits robust convergence on all datasets. Notably, the method converges within 60 iterations for most dataset, reflecting an excellent overall convergence speed.

\begin{figure*}[h] 
\centering
\subfigure[Coil]
 {
 \includegraphics[width=1.23in,height=0.92in]{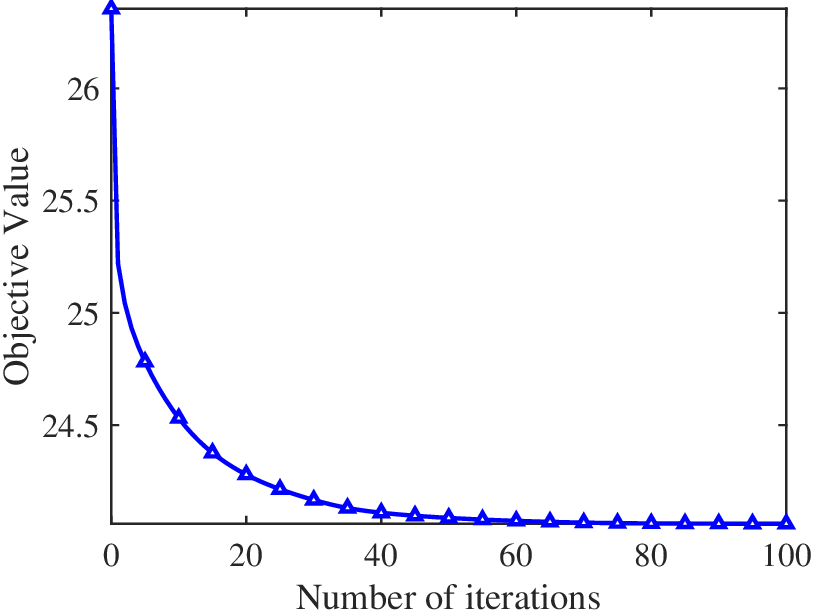}
 \label{fig:OBJCOIL20}
 }
 \subfigure[Wiki]
 {
 \includegraphics[width=1.23in,height=0.92in]{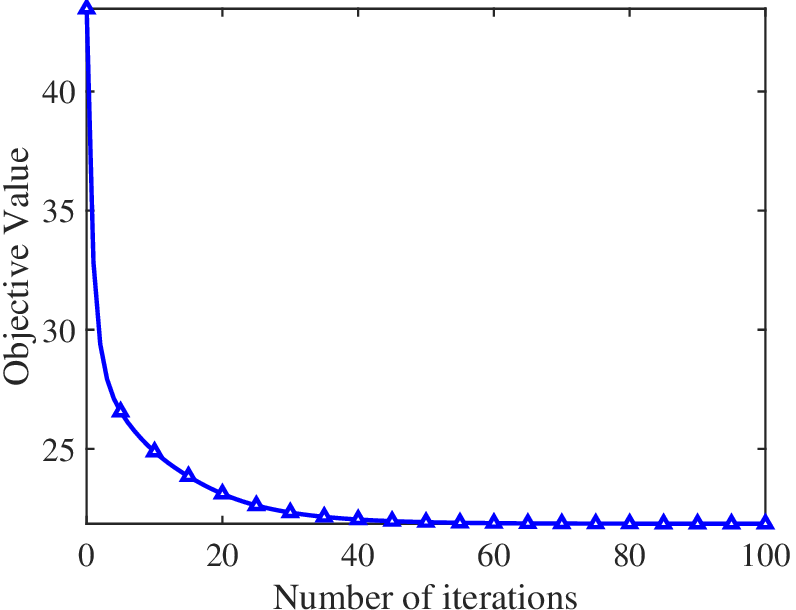}
 \label{fig:OBJWiKi}
 }
 \subfigure[USPS]
 {
 \includegraphics[width=1.23in,height=0.92in]{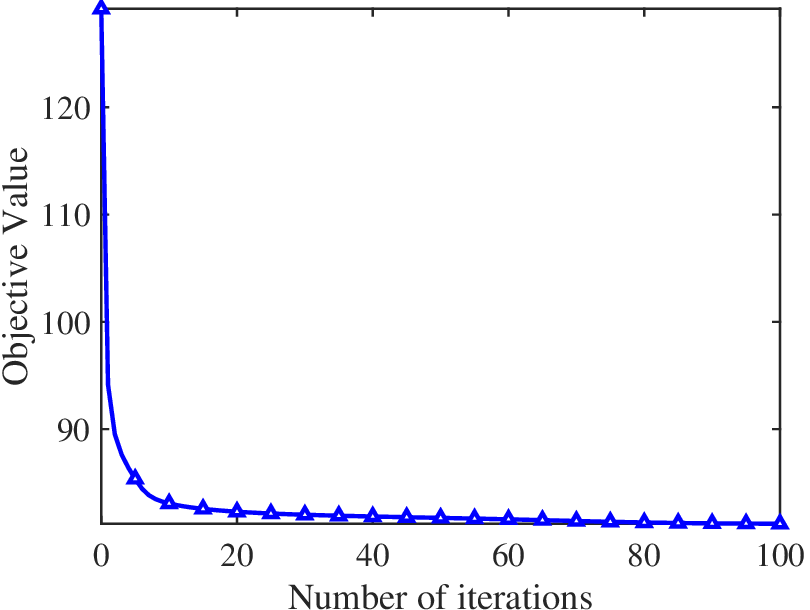}
 \label{fig:OBJUSPS}
 }
 \subfigure[Reuters]
 {
 \includegraphics[width=1.23in,height=0.92in]{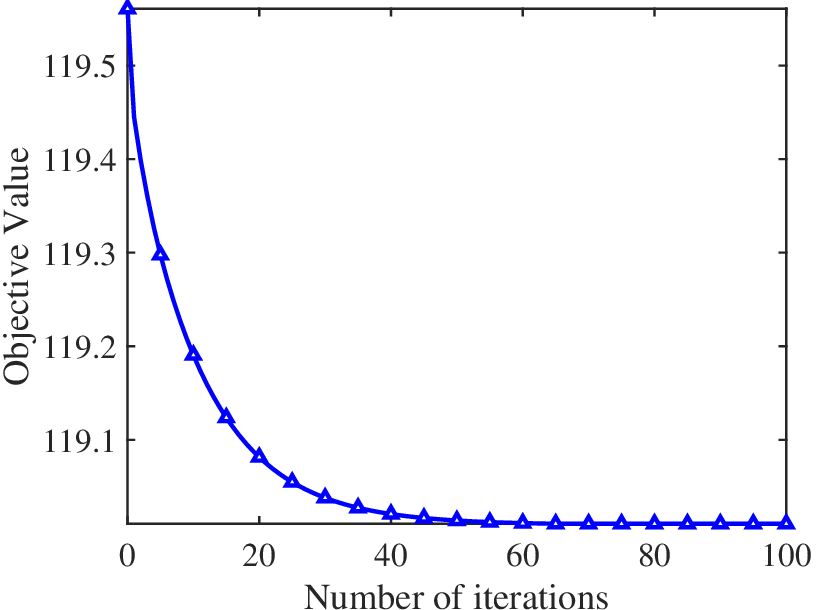}
 \label{fig:OBJReuters}
 }
  \subfigure[NoisyMNIST]
 {
 \includegraphics[width=1.23in,height=0.92in]{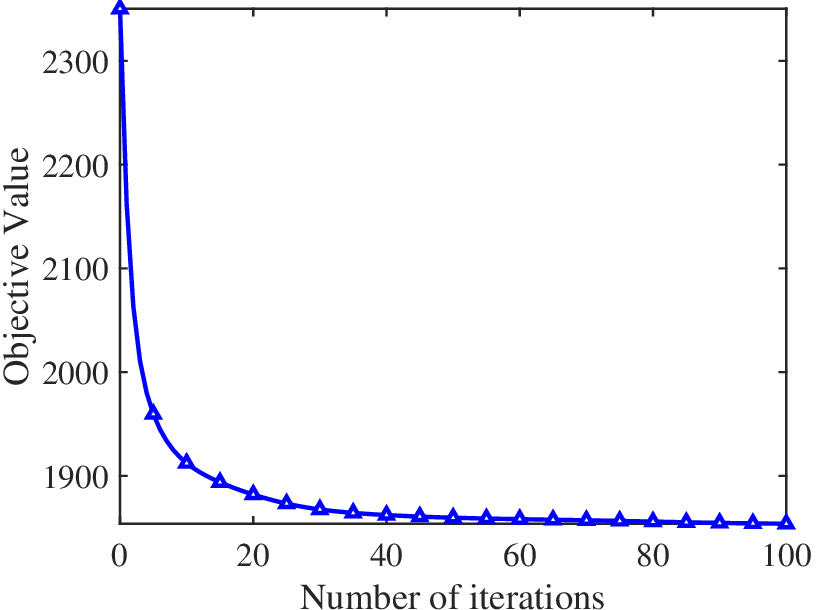}
 \label{fig:OBJNoisyMNIST}
 }
  \subfigure[XMedia]
 {
 \includegraphics[width=1.23in,height=0.92in]{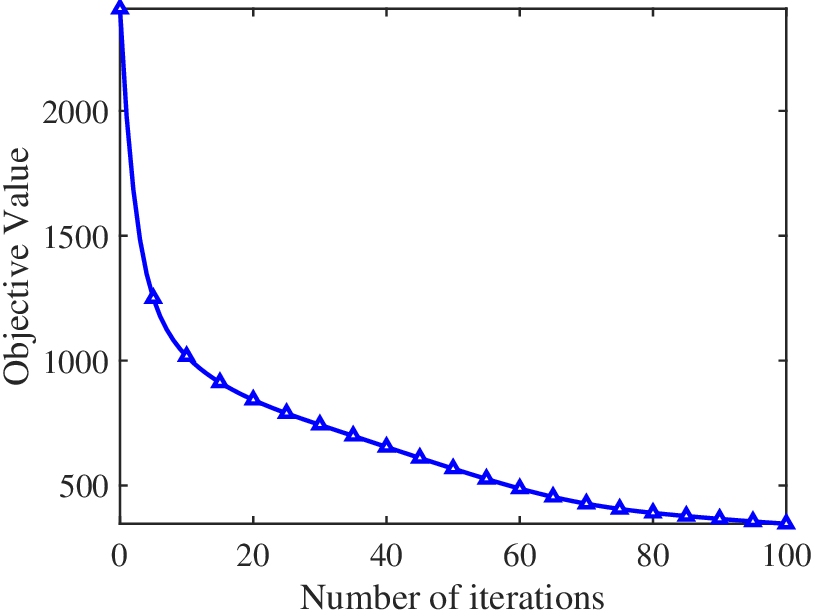}
 \label{fig:OBJXMedia}
 }
 \subfigure[Cifar10]
 {
 \includegraphics[width=1.23in,height=0.92in]{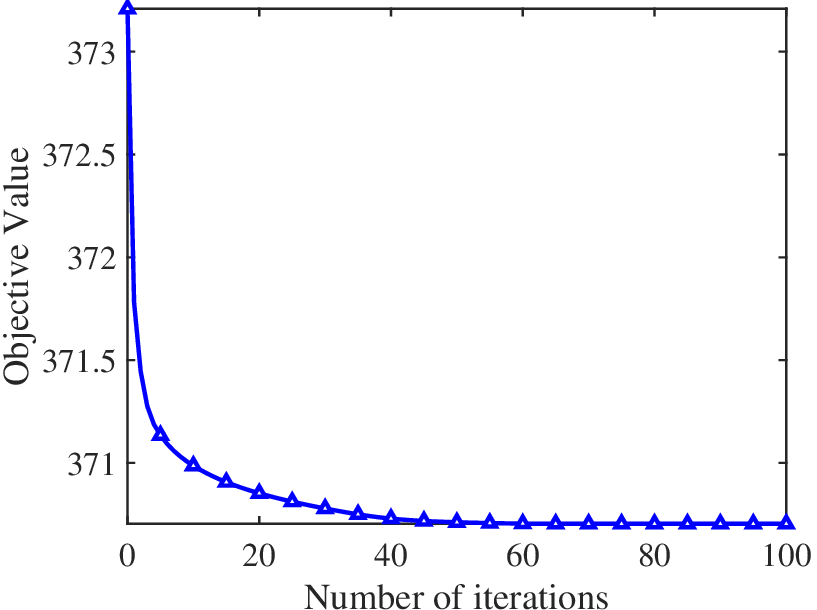}
 \label{fig:OBJcifar10}}
 \subfigure[Cifar100]
 {
 \includegraphics[width=1.23in,height=0.92in]{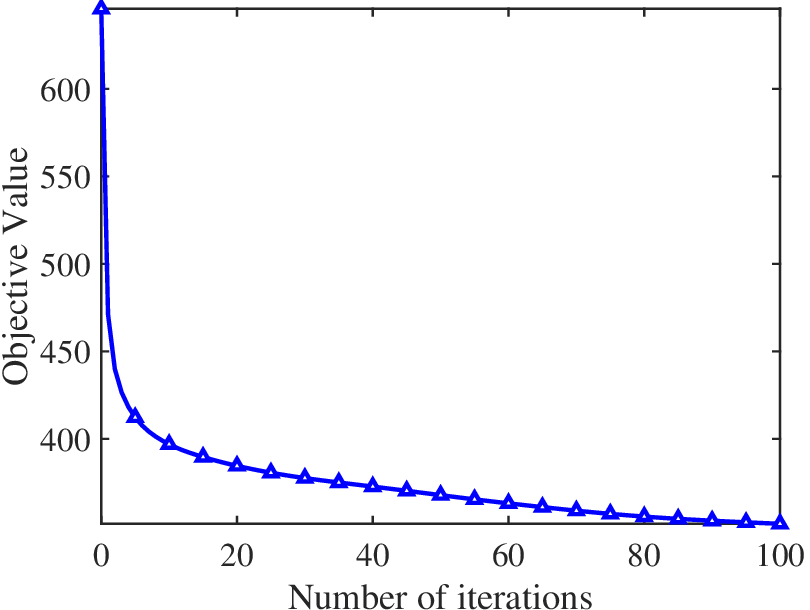}
 \label{fig:OBJMnist4}
 }
 \subfigure[Mnist]
 {
 \includegraphics[width=1.23in,height=0.92in]{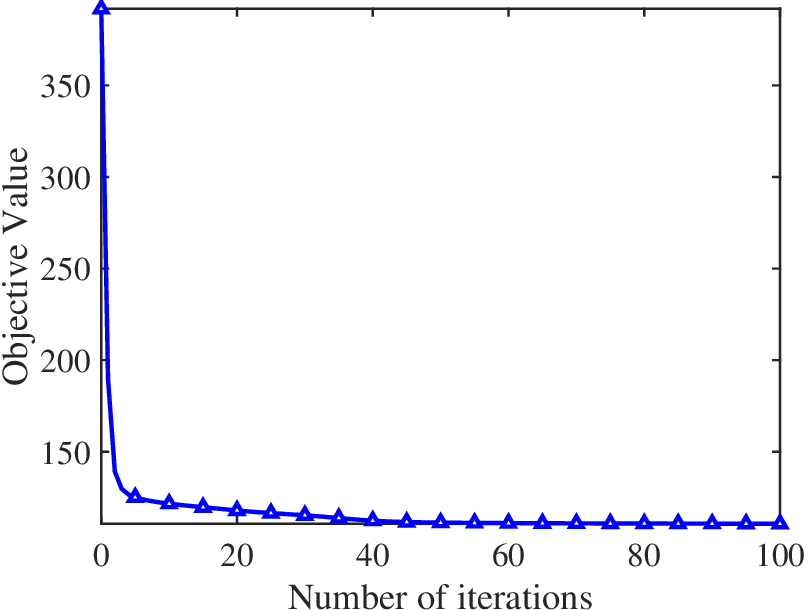}
 \label{fig:OBJMNIST}
 }
 
 \caption{Convergence curves of OMCAL on different real-world multi-view datasets: (a) Coil; (b) Wiki; (c) USPS; (d) Reuters; (e) NoisyMNIST; (f) XMedia; (g) Cifar10; (h) Cifar100; (i) Mnist.}
 \label{fig:Convergencecurve} 
 \end{figure*}

\subsection{Parameter Sensitive}

This subsection details the experiments we conducted to assess the influence of hyper-parameters. For OMCAL, there are three hyper-parameters: the trade-off parameters $\beta$, $\gamma$ and the number of anchors $m$. 

\begin{figure}[!ht] 
\centering
\subfigure[USPS]
 {
 \includegraphics[width=1.62in,height=1.35in]{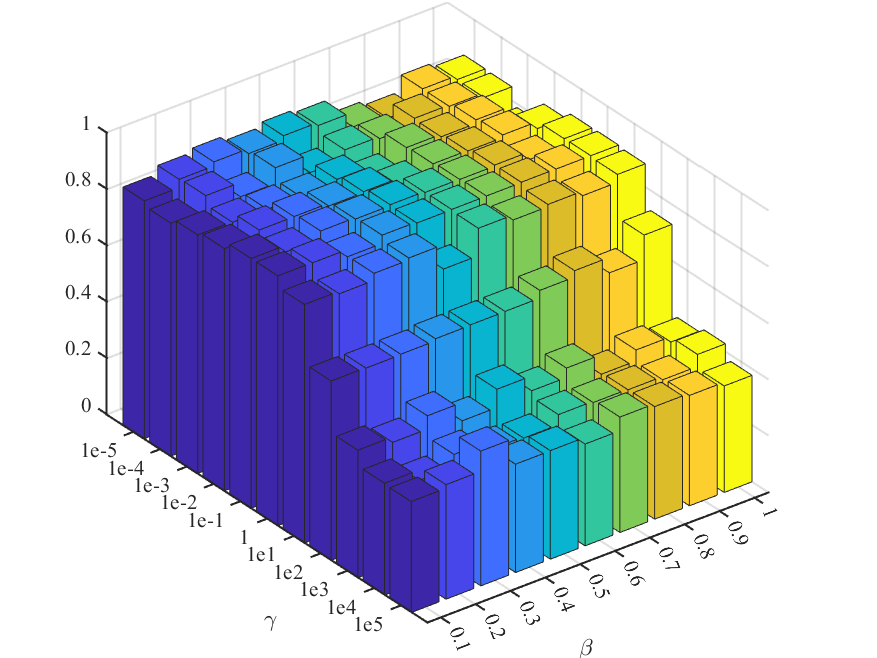}
 \label{fig:ouhe_usps__acc}
 }
 \subfigure[Cifar10]
 {
 \includegraphics[width=1.62in,height=1.35in]{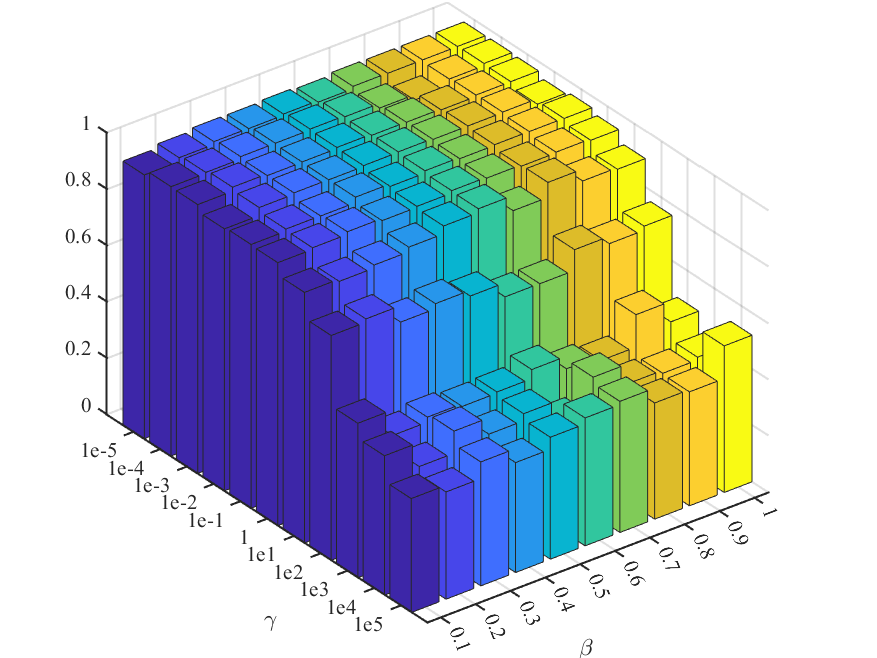}
 \label{fig:ouhe_wiki__acc}
 }
 \caption{Clustering accuracy with different $\beta$ and $\gamma$ on two datasets: (a) USPS; (e) Cifar10.}
 \label{fig:ouhe} 
 \end{figure}

For the traded-off parameters $\beta$ and $\gamma$, we fixed $m$ and varied them using two datasets USPS and Cifar10 as examples. Specifically, We varied $\gamma$ from $1e-5$ to $1e5$, and varied $\beta$ from $0.1$ to $1$. To illustrate their coupling relationship, we present the influence of both parameters on clustering effectiveness in Fig \ref{fig:ouhe}. The results indicate that $\beta$ has a minimal effect on clustering performance while increasing $\gamma$ negatively impacts clustering effectiveness. This observation suggests that while adjusting $\beta$ may not significantly alter the effectiveness, careful consideration of $\gamma$ is crucial, as it can lead to a decline in clustering effectiveness if set too high. Therefore, we recommend selecting an value for $\gamma$ from $1e-5$ to $1e0$.

For the number of anchors $m$, we fixed $\beta$ and $\gamma$ and varied $m$ using two datasets USPS and Cifar10 as examples. Generally, the number of anchors $m$ should exceed the number of clusters $c$. So we adjusted $m$ from $c$ to $c+150$. As shown in Fig. \ref{fig:anchor}, we observed that the clustering performance shows significantly improvement as the number of anchors grows when it is small. This enhancement is attributed to the anchor graph containing more information as the number of anchors rises. However, beyond a certain number, further increasing the number of anchors yields diminishing returns on clustering performance and may even lead to a slight decline. Regarding computational cost, Fig. \ref{fig:anchor_time} illustrates that the computational expense of OMCAL escalates considerably with an increase in $m$. This rise in cost can be associated to two principal factors: first, constructing the anchor graph becomes more time-consuming with an increasing number of anchors; second, a larger $m$ leads to a higher dimensional anchor graph, which intensifies the computational complexity during the optimization process. Therefore, we recommend choosing the number of anchors between $c+10$ and $c+150$. It should be noted that, for data with a normal number of samples and clusters (e.g., USPS), the number of anchors should be closer to $c+10$ to prevent computational cost and ensure stable clustering performance. For large-scale data, a higher number of anchors (closer to $c+150$) is recommended to capture finer data structures and improve clustering precision. 

 \begin{figure}[!hb]
 \centering
 \subfigure[USPS]
 {
 \includegraphics[width=1.62in,height=1.35in]{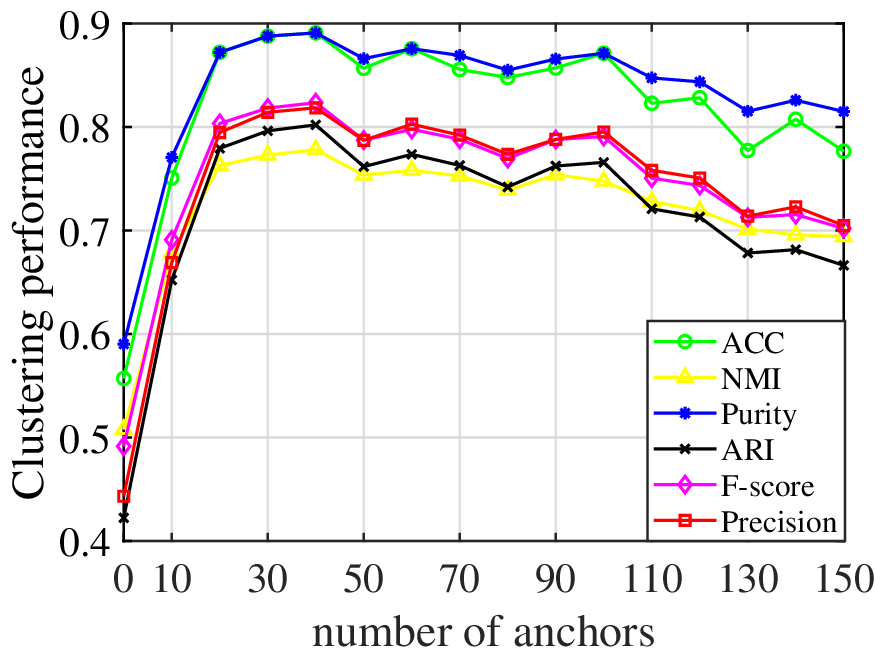}
 }
 \hfil
 \subfigure[Cifar10]
 {
 \includegraphics[width=1.62in,height=1.35in]{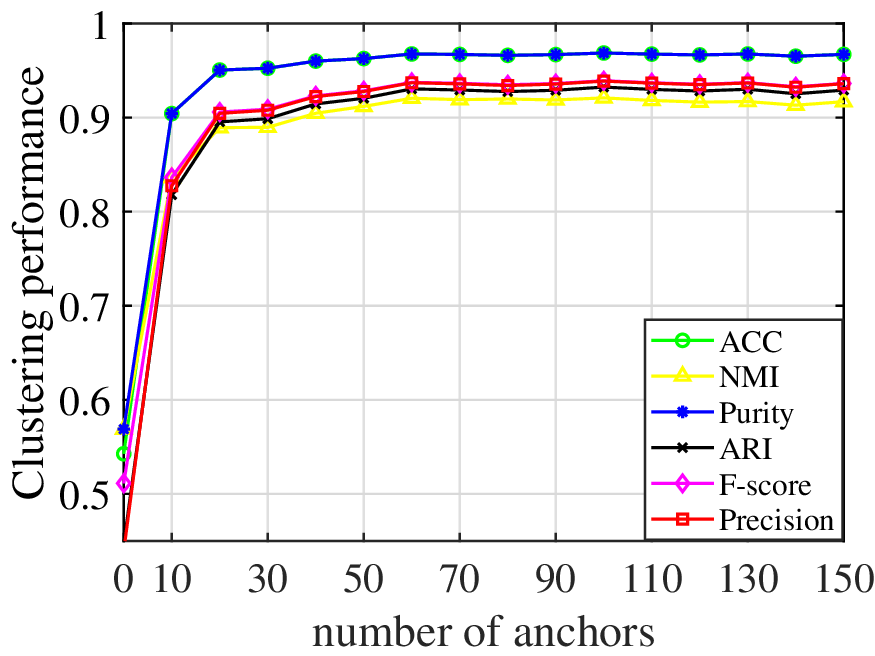}
 }
 \caption{Clustering result and computational time with different $m$ on two datasets: (a) USPS; (b) Cifar10 } 
 \label{fig:anchor}
\end{figure}

\begin{figure}[!ht]
 \centering
 \subfigure[USPS]
 {
 \includegraphics[width=1.62in,height=1.35in]{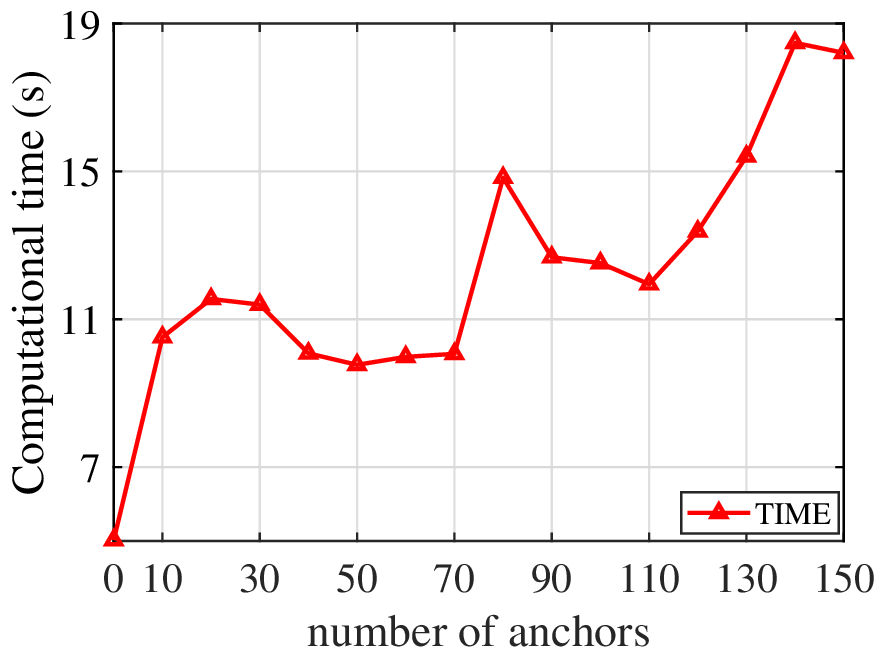}
 }
 \hfil
 \subfigure[Cifar10]
 {
 \includegraphics[width=1.62in,height=1.35in]{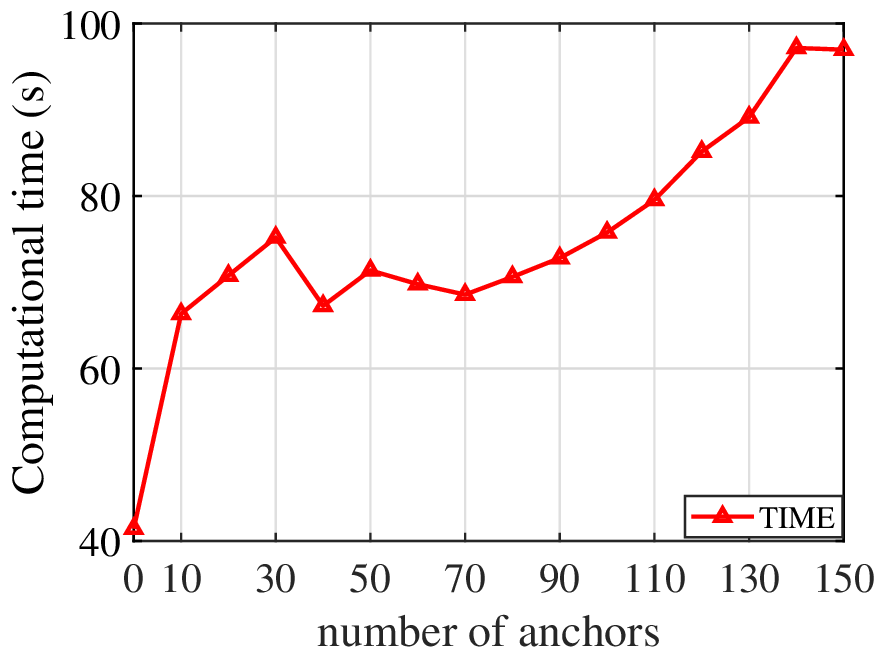}
 }
 \caption{Clustering computational time with different $m$ on two datasets: (a) USPS; (b) Cifar10} 
 \label{fig:anchor_time}
\end{figure}

\subsection{Visualization}
To demonstrate the effectiveness of our proposed consensus anchor graph learning strategy, we demonstrate the visualization of our strategy using the datasets USPS and Cifar10 as examples. Specifically, we recorded both the original consensus anchor graph and the consensus anchor graph learned through OMCAL for these two datasets and performed a visual comparison of the two anchor graphs. To reveal the similarity information among samples, we need to convert the consensus anchor graphs before and after learning into full-sample graphs. Specifically, given an anchor graph $\mathbf{S} \in \mathbb{R}_{+}^{n \times m}$, we generate the full-sample graph according to the method outlined in \cite{liu2010large}: $\mathbf{B} = \mathbf{SD}^{-1}\mathbf{S}^{\top}$. where $\mathbf{D} \in \mathbb{R}^{m\times m}$ is a diagonal matrix defined as $\mathbf{D}_{jj}=\sum_{i=1}^{n}\mathbf{S}_{ij}$.

\begin{figure}[!h]
 \centering
 \subfigure[Original consensus structure]
 {
 \includegraphics[width=1.62in,height=1.35in]{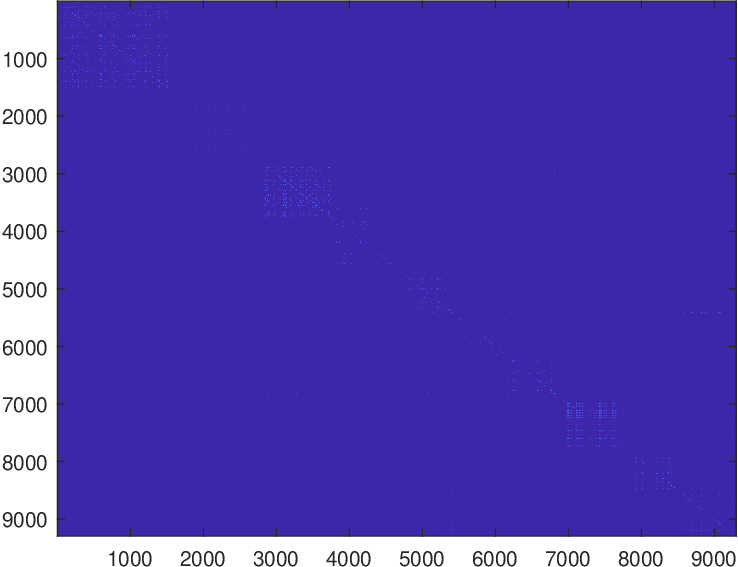}
 }
 \hfil
 \subfigure[Learned consensus structure]
 {
 \includegraphics[width=1.62in,height=1.35in]{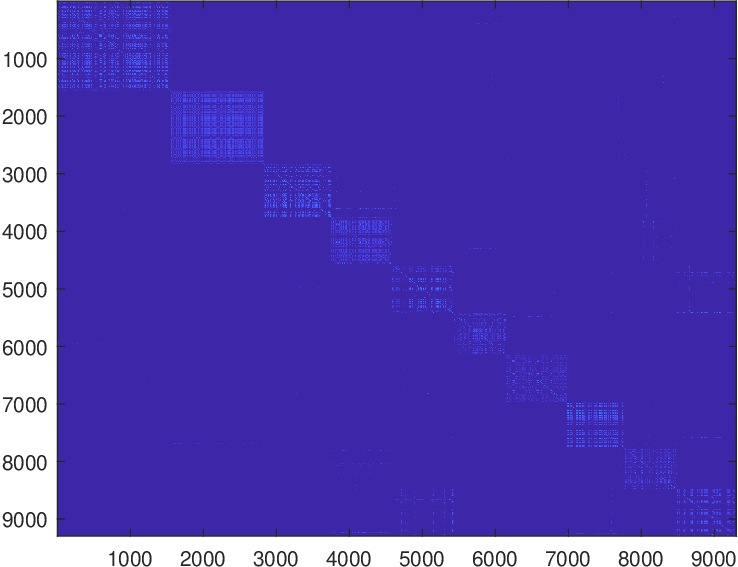}
 }
 \caption{Consensus structure visualization on USPS: (a) Before learning; (b) After learning} 
 \label{fig:USPS anchor graph visualization}
\end{figure}

\begin{figure}[!h]
 \centering
 \subfigure[Original consensus structure]
 {
 \includegraphics[width=1.60in,height=1.46in]{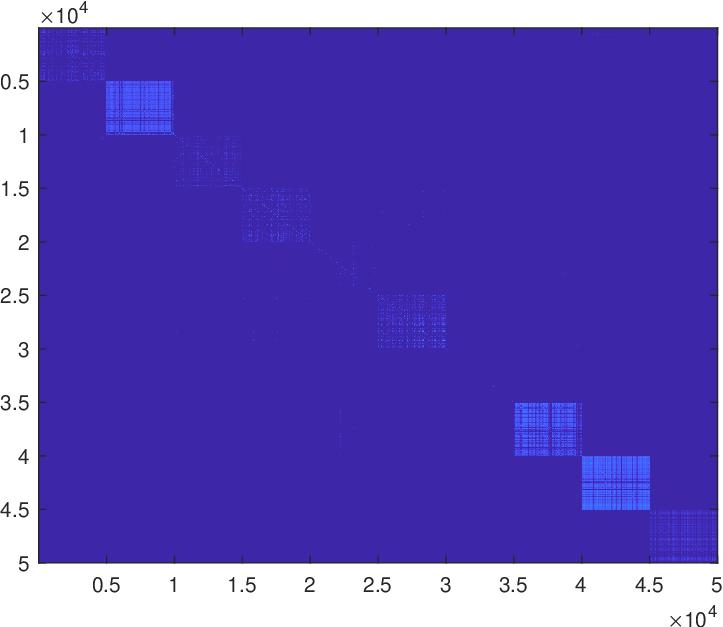}
 }
 \hfil
 \subfigure[Learned consensus structure]
 {
 \includegraphics[width=1.60in,height=1.46in]{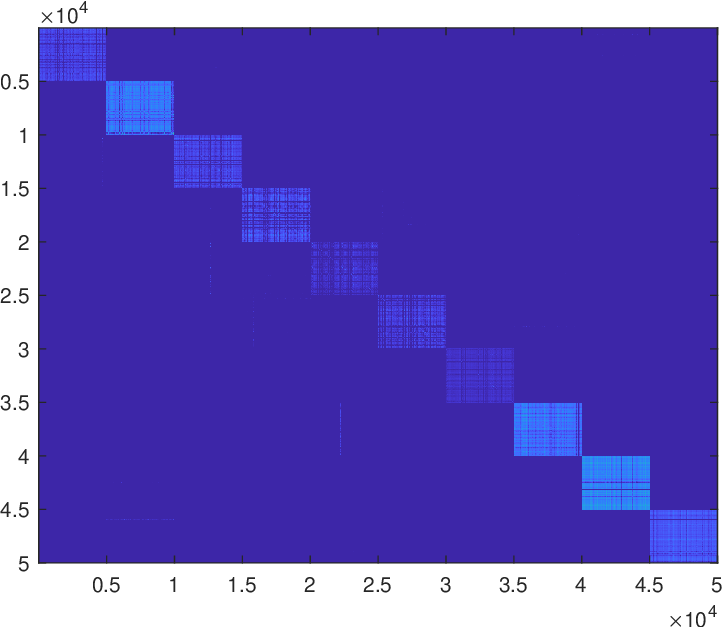}
 }
 \caption{Consensus structure visualization on Cifar10: (a) Before learning; (b) After learning} 
 \label{fig:cifar10 anchor graph visualization}
\end{figure}

As shown in Fig. \ref{fig:USPS anchor graph visualization} and Fig \ref{fig:cifar10 anchor graph visualization}, the consensus anchor graphs for the USPS and Cifar10 datasets exhibit clearer connectivity structures after learning. These results demonstrate that our proposed anchor graph learning strategy, based on the kernel norm, effectively captures the underlying consensus structure across different views. 
The enhanced clarity in connectivity suggests a more precise representation of the relationships among samples, thereby improving the overall clustering performance of our method.

\subsection{Transfer Our Method to Another Method}
Similar to OMCAL, SMC also fuses multi-view anchor graphs into a consensus anchor graph. In this subsection, we transferred the OMCAL’s consensus anchor graph learning and clustering indicator acquisition framework to SMC, creating a new method named SMC-OCL. Additionally, we transferred the anchor graph construction method in Eq. (\ref{eq:construct anchor graph}) to SMC, resulting in a new method named SMC-AGC. Since SMC runs out of memory on the Cifar10 dataset, we report the performance of the three methods on Wiki and USPS.

As shown in Fig. (\ref{fig:Transfer Our Method to SMC}), on the WiKi and USPS datasets, SMC-AGC and SMC have similar performance under six evaluation metrics. Furthermore, SMC-OCL shows a clear advantage in clustering effectiveness, highlighting the benefit of integrating low-rank consensus anchor graph learning and clustering indicator acquisition into a unified framework.

\begin{figure}[h]
 \centering
 \subfigure[WiKi]
 {
 \includegraphics[width=1.60in,height=1.46in]{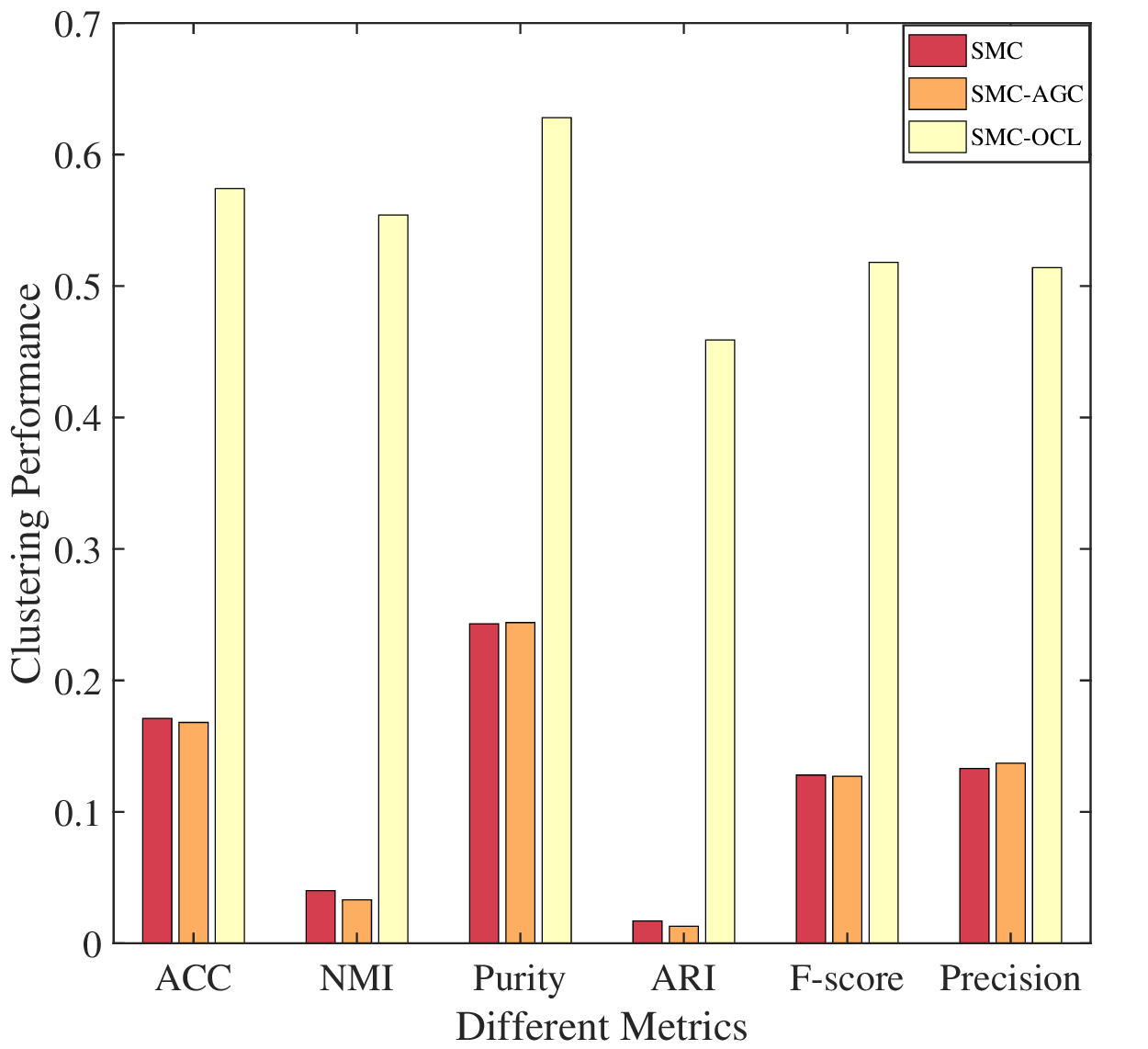}
 }
 \hfil
 \subfigure[USPS]
 {
 \includegraphics[width=1.60in,height=1.46in]{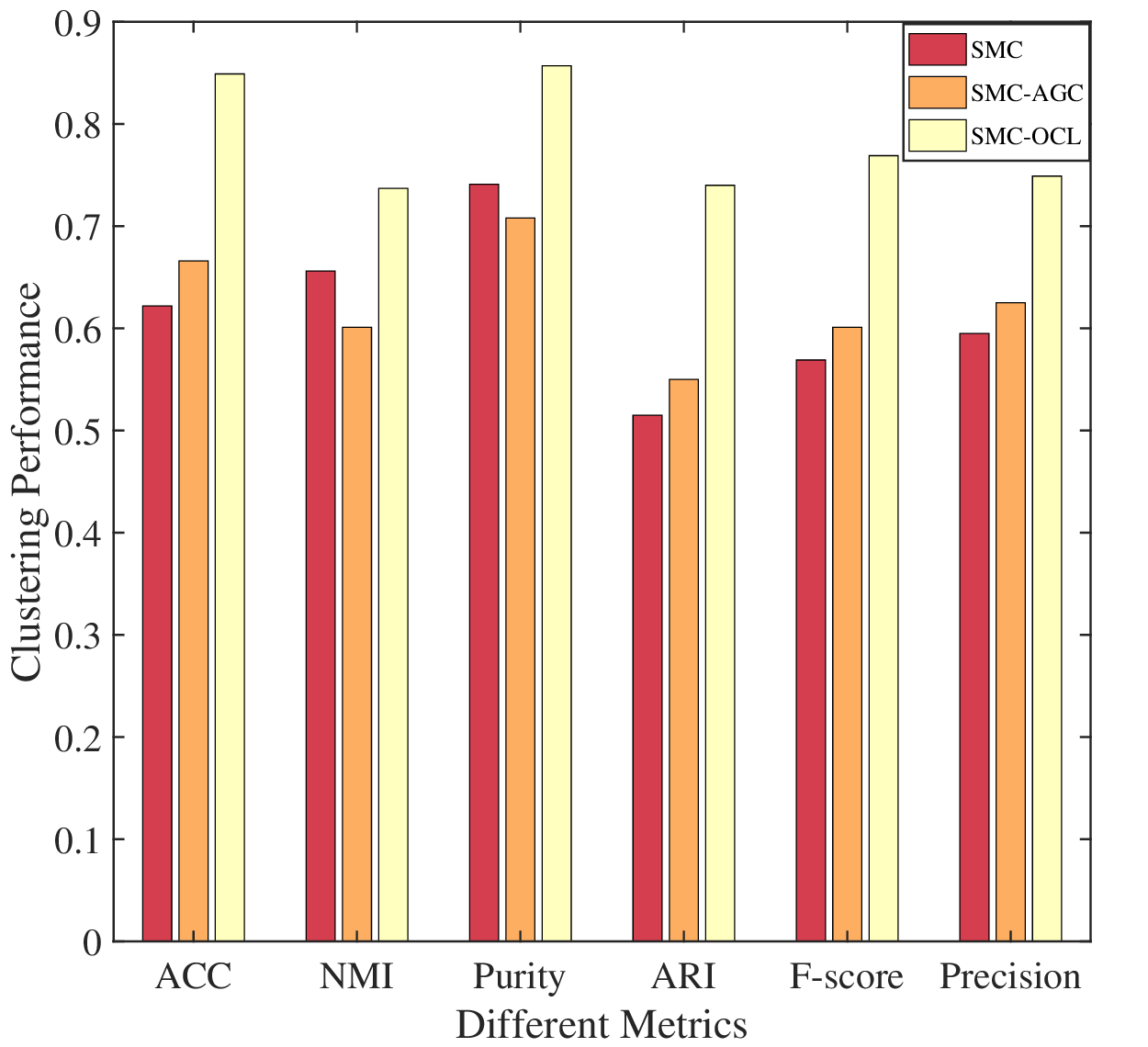}
 }
 \caption{\centering Clustering Performance on two datasets: (a) WiKi; (b) USPS}  
 \label{fig:Transfer Our Method to SMC}
\end{figure}

\begin{table}[h]
\begin{adjustbox}{center}
\begin{threeparttable}
\caption{\label{Tab:detail of single datasets}The Detail Of Different Sulti-view Datasets}
 \centering
 \makeatletter 
 \begin{tabular*}{8.5cm}{@{\extracolsep{\fill}}lccc}
 \toprule[1pt]
 Dataset &  Samples& Clusters & Features  \\
 \midrule[0.2pt]
TDT2  & 1938 & 20 & 36771  \\
Scotus & 6400 & 13 & 126405 \\
Cal101 & 8641 & 101 & 256 \\
\bottomrule[1pt]
 \end{tabular*}
\end{threeparttable}
\end{adjustbox}
\end{table}

\subsection{Clustering on Single-View Data}

In this subsection, we extend OMCAL to the single-view data, the new method is named as OCAL (One-step Clustering with Adaptive Low-rank Anchor-graph Learning). For single-view data, OCAL no longer learns low-rank consensus anchor graphs, as in OMCAL. Instead, it directly uses the anchor graph to learn a high-quality low-rank anchor graph. The objective function of OCAL can be written as follows:
\begin{equation}
    \label{eq:objective function_OCAL}
 \begin{split}
 \min\limits_{\mathbf{Z},\mathbf{F},\mathbf{G}}\left \| \mathbf{Z}  - \mathbf{S} \right \|_{\rm{F}}^{2}  &+ \beta \left \| \mathbf{Z} \right \| _{\ast } + \gamma \left \| \mathbf{Z}-\mathbf{FG}^{\top} \right \|_{\rm{F}}^{2}
\\
 s.t. &\mathbf{G}^{\top}\mathbf{G}=\mathbf{I}, \mathbf{F}\ge 0 
\end{split}
\end{equation}
where $\mathbf{S}$ represnets original anchor graph constructed from single-view data and the meanings of other variables are the same as in OMCAL. The optimization method of Eq. (\ref{eq:objective function_OCAL}) is similar to that of Algorithm \ref{algor for OMCAL}, except that OCAL does not need to optimize $\boldsymbol{\alpha}$. Therefore, the specific optimization process will not be described in detail in this subsection.

\begin{table*}[t]
 \begin{adjustbox}{center}
\begin{threeparttable}
\caption{\label{Tab:result_single}Clustering Results Comparison on Single-View Datasets}
 \centering
 \setlength{\tabcolsep}{1.2mm}
 \makeatletter
 \begin{tabular*}{17cm}{@{\extracolsep{\fill}}llccccccc}
 \toprule[1pt]
Datasets & Methods & ACC & NMI & Purity & ARI & F-score & Precision & Time \\
\midrule[0.2pt]
\multirow{8}{*}{TDT2} & NMF  & 0.816 (0.066) & 0.867 (0.033) & 0.861 (0.039) & 0.793 (0.063) & 0.804 (0.059) & 0.808 (0.072) & \underline{61.022} \\
 & SC & 0.093 (0.001) & 0.022 (0.000) & 0.104 (0.000) & 0.000 (0.000) & 0.106 (0.000) & 0.056 (0.000) & 376.094 \\ 
 & FNMTF & 0.656 (0.068) & 0.774 (0.037) & 0.713 (0.054) & 0.596 (0.066) & 0.621 (0.061) & 0.568 (0.072) & 1704.627  \\
 & BKM & 0.328 (0.078) & 0.383 (0.120) & 0.329 (0.078) & 0.193 (0.128) & 0.270 (0.110) & 0.165 (0.082) & 1057.216 \\
 & ECPCS & \underline{0.942 (0.025)} & \textbf{0.942 (0.010)} & \underline{0.950 (0.018)} &\textbf{0.921 (0.016)}  & \textbf{0.926 (0.015)} & \underline{0.930 (0.011)} & 88.512 \\
 & KMM & 0.911 (0.038) & 0.912 (0.025) & 0.914 (0.036) & 0.844 (0.062) & 0.854 (0.058) & 0.777 (0.098) & 153.049 \\
 & CGFKM & 0.880 (0.051) & 0.915 (0.023) & 0.911 (0.031) & 0.848 (0.061) & 0.857 (0.057) & 0.835 (0.087) & 4018.490 \\
 & \textbf{OCAL} & \textbf{0.952 (0.015)} & \underline{0.940 (0.011)} & \textbf{0.954 (0.011)} & \underline{0.920 (0.016)} & \underline{0.925 (0.016)} & \textbf{0.934 (0.017)} & \textbf{7.766 } \\
\midrule[0.2pt]
\multirow{8}{*}{Scotus} & NMF & 0.210 (0.020) & \underline{0.083 (0.013)} & \underline{0.331 (0.023)} & \underline{0.060 (0.026)} & 0.159 (0.024) & \textbf{0.218 (0.032)} & 959.444 \\
 & SC & 0.206 (0.005) & 0.004 (0.001) & 0.216 (0.001) & 0.000 (0.001) & 0.226 (0.016) & 0.145 (0.000) & \underline{751.621} \\ 
& FNMTF & 0.166 (0.020) & 0.060 (0.023) & 0.296 (0.018) & 0.020 (0.009) & 0.139 (0.010) & 0.165 (0.009) & 13961.907 \\
 & BKM & 0.219 (0.002) & 0.007 (0.001) & 0.220 (0.001) & 0.001 (0.000) & 0.251 (0.000) & 0.145 (0.000) & 5719.944 \\
 & ECPCS & OM & OM & OM & OM & OM & OM & OM \\
 & KMM & \underline{0.220 (0.001)} & 0.012 (0.002) & 0.223 (0.001) & 0.002 (0.001) & \underline{0.253 (0.001)} & 0.145 (0.001) & 1428.095 \\
 & CGFKM & --- & --- & --- & --- & --- & --- & --- \\
& \textbf{OCAL} & \textbf{0.322 (0.052)} & \textbf{0.171 (0.081)} & \textbf{0.393 (0.087)} & \textbf{0.097 (0.058)} & \textbf{0.266 (0.018)} & \underline{0.214 (0.044)} & \textbf{89.362} \\
\midrule[0.2pt]
\multirow{8}{*}{Caltech101} & NMF & 0.165 (0.013) & 0.362 (0.009) & 0.286 (0.010) & 0.132 (0.021) & 0.151 (0.022) & 0.203 (0.020) & \underline{18.606 }\\
 & SC & 0.225 (0.032) & 0.260 (0.041) & 0.266 (0.046) & 0.127 (0.040) & 0.170 (0.037) & 0.099 (0.025) & 22.080 \\ 
 & FNMTF & \underline{0.304 (0.008) }& \underline{0.535 (0.006)} & 0.465 (0.005) & \textbf{0.294 (0.024)} & \textbf{0.308 (0.025)} & \textbf{0.440 (0.007)} & 211.582 \\
 & BKM & 0.190 (0.004) & 0.164 (0.004) & 0.190 (0.004) & 0.077 (0.007) & 0.125 (0.006) & 0.068 (0.004) & 144.983 \\
 & ECPCS & 0.291 (0.002) & 0.531 (0.001) & \textbf{0.469 (0.002)} & 0.246 (0.003) & 0.261 (0.003) & 0.374 (0.003) & 35.836 \\ 
 & KMM & 0.243 (0.033) & 0.251 (0.056) & 0.303 (0.032) & 0.026 (0.019) & 0.079 (0.017) & 0.042 (0.010) & 223.692 \\
 & CGFKM & 0.273 (0.008) & 0.522 (0.003) & 0.463 (0.004) & 0.197 (0.010) & 0.214 (0.010) & 0.290 (0.008) & 572.590 \\
 & \textbf{OCAL} & \textbf{0.311 (0.013)} & \textbf{0.537 (0.007)} & \underline{0.468 (0.003)} & \underline{0.286 (0.029)} & \underline{0.300 (0.029)} & \underline{0.428 (0.023)} & \textbf{12.235} \\
\bottomrule[1pt]
\end{tabular*}
\begin{tablenotes}
\footnotesize
\item *The best results are shown in bold, while the second-best results are indicated with underlining
 \end{tablenotes}
\end{threeparttable}
 \end{adjustbox}
\end{table*}

To evaluate the clustering  performance of OCAL, we conducted experiments on three widely used single-view datasets: TDT2 \cite{cai2009probabilistic}, Scotus \cite{chalkidis2022lexglue} and Caltech101 \cite{fei2004learning}. The details of datasets are summary in Table \ref{Tab:detail of single datasets}. We compared OCAL with seven state-of-the-art single-view clustering methods, including NMF \cite{lee1999learning}, SC \cite{ng2001spectral}, FNMTF \cite{wang2011fast}, BKM \cite{han2017bilateral}, ECPCS \cite{huang2018enhanced} and CGFKM \cite{du2024k}. For a fair comparison, we use the open-source code or original code and carefully tuned the parameters for each method. We select the best performing parameter settings for each method. Each method was executed five times with its optimal parameter settings. ‘OM’ is recorded when a method runs out of memory, and ‘---’ is recorded when a method runs for more than 6 hours.

As shown in Table \ref{Tab:result_single}, in terms of clustering efficiency, OCAL has the shortest running time on all three datasets, highlighting its superior clustering efficiency. Although OCAL is a single-view clustering method, the reason why it canachieve excellent clustering effciency is similar to OMCAL. Notably, while OCAL employs non-negative matrix decomposition to avoid the post-processing step for getting clustering indicator after anchor graph analysis, it still outperforms other clustering methods based on matrix decomposition in terms of efficiency. This is because traditional clustering methods based on matrix decomposition are more sensitive to the number of features, whereas OCAL effectively mitigates the impact of high-dimensional features on computational efficiency by constructing low-dimensional anchor graph. 

In terms of clustering effectiveness, OCAL ranks first or second in all the metrics for measuring clustering performance, underscoring its superior clustering capability. Similar to OMCAL, OCAL avoids information loss by eliminating the post-processing step when getting clustering indicator, enhancing clustering accuracy. The key difference is that OMCAL’s one-step low-rank anchor graph learning framework operates on anchor graphs from multi-view, whereas OCAL applies one-step low-rank learning framework directly to single-view anchor graph. Nonetheless, this strategy remains effective because the joint optimization of low-rank learning and clustering indicator acquisition helps remove redundant information and extract a clearer underlying clustering structure.

\section{Conclusion}
In this paper, we have proposed a novel one-step multi-view clustering method with adaptive low-rank anchor-graph learning (OMCAL). To address the challenge that the acquisition of consensus anchor graphs is faced by redundant information and noise in the current large-scale multi-view clustering methods based on anchor graphs, we establish a consensus anchor graph learning strategy based on kernel norm. To mitigate the information and efficiency loss associated with the additional post-processing step, we utilize orthogonal matrix decomposition to directly analyze the consensus anchor graph and derive the clustering indicator. For optimization, we propose a fast method that requires only a small amount of computational efforts to solve OMCAL. Finally, we conduct extensive experiments on a variety of real-world datasets, including large-scale datasets, to demonstrate the efficiency and effectiveness of OMCAL.

\bibliographystyle{IEEEtran}
\bibliography{IEEEabrv,ref}

\begin{IEEEbiography}[{\includegraphics[width=1in,height=1.25in,clip,keepaspectratio]{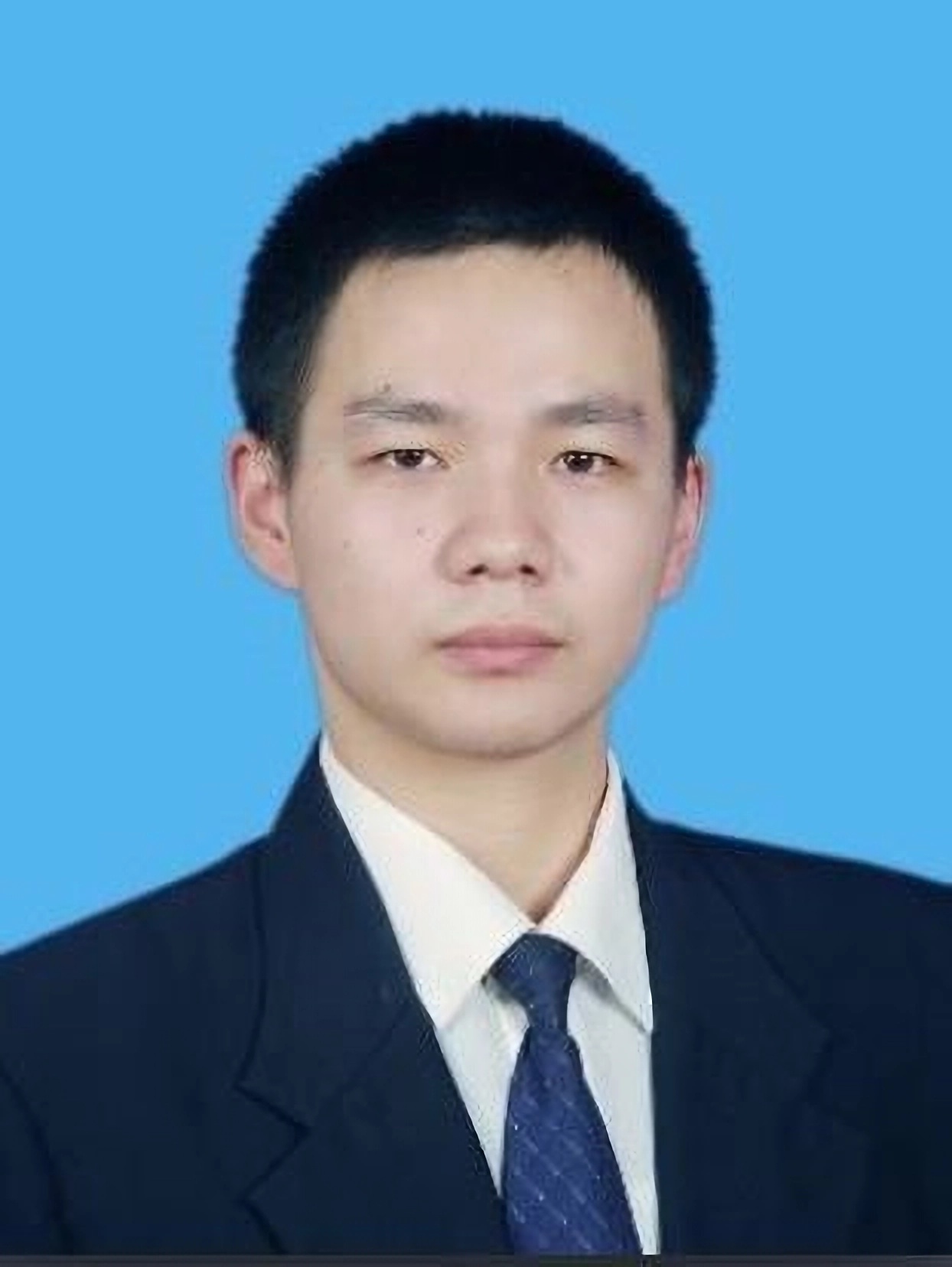}}]{Zhiyuan Xue}
received the B.S. degree from Northwestern Polytechnical University, M.S. degree in control science and engineering from Xi\rq an jiaotong University. Currently, he is pursuing his Ph.D. degree at the Institute of Artificial Intelligence and Robotics and College of Artificial Intelligence, Xi\rq an Jiaotong University. His current research interests include machine learning and computer vision.
\end{IEEEbiography}

\begin{IEEEbiography}[{\includegraphics[width=1in,height=1.25in,clip,keepaspectratio]{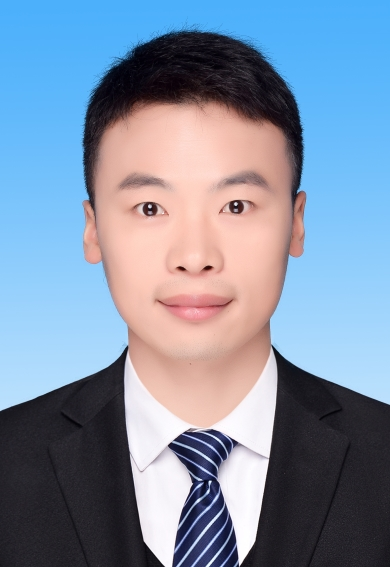}}]{Ben Yang}(Member, IEEE)
received the B.S. degree in automation science and technology and the Ph.D. degree in control science and engineering from Xi’an Jiaotong University, Shaanxi, China, in 2016 and 2023, respectively. He visited the School of Electrical and Electronic Engineering, Nanyang Technological University, Singapore, from 2021 to 2022. He is currently an assistant professor with the Institute of Artificial Intelligence and Robotics and the College of Artificial Intelligence, Xi’an Jiaotong University. His research interests are in data mining, machine learning, image processing, and pattern recognition.
\end{IEEEbiography}

 \begin{IEEEbiography}[{\includegraphics[width=1in,height=1.25in,clip,keepaspectratio]{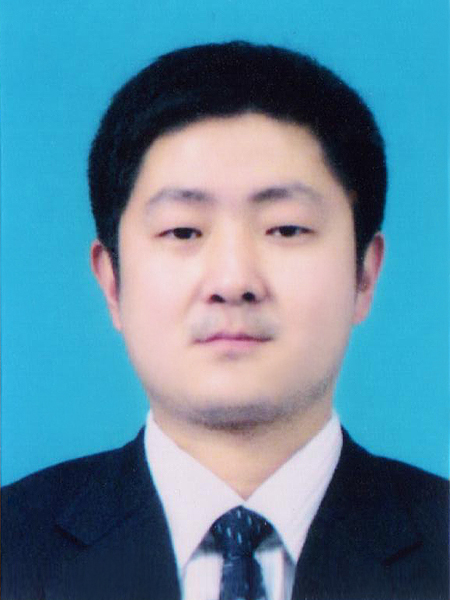}}]{Xuetao Zhang}(Member, IEEE)
received the B.S. degree in information engineering, M.S. degree and Ph.D. degree in pattern recognition and intelligence systems from Xi\rq an Jiaotong University, China in 2003, 2006, and 2012 respectively. He visited the Department of Brain and Cognitive Sciences, Massachusetts Institute of Technology from 2009 to 2010. He is currently an associate professor of the Institute of Artificial Intelligence and Robotics in Xi\rq an Jiaotong University. His research interests include computer vision, human vision, and machine learning.
\end{IEEEbiography} 

\begin{IEEEbiography}[{\includegraphics[width=1in,height=1.25in,clip,keepaspectratio]{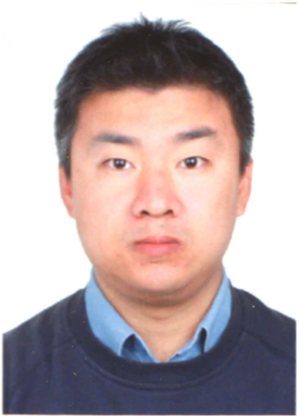}}]{Fei Wang}(Member, IEEE)
received his B.S. degree in electronics from Northwest University, M.S. degree in Communication and information systems from Xi\rq an Institute of Optics and Precision Mechanics, Chinese Academy of Sciences, Ph.D. degree in pattern recognition and intelligence system from Xi\rq an Jiaotong University, China in 1998, 2002 and 2009 respectively. He visited the North Carolina State University in the USA from 2012 to 2013. He is currently a professor of the Institute of Artificial Intelligence and Robotics in Xi\rq an Jiaotong University. His research interests include computer vision and intelligent systems.
\end{IEEEbiography} 

\begin{IEEEbiography}[{\includegraphics[width=1in,height=1.25in,clip,keepaspectratio]{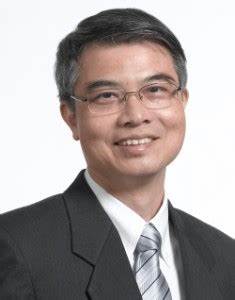}}]{Zhiping Lin}(Senior Member, IEEE) 
received the B.Eng. degree in control engineering from the South China Institute of Technology, Guangzhou, China, in 1982, and the Ph.D. degree in information engineering from the University of Cambridge, Cambridge, U.K., in 1987. He worked with the University of Calgary, Calgary, AB, Canada, Shantou University, Shantou, China, and DSO National Laboratories, Singapore, before joining the School of Electrical and Electronic Engineering, Nanyang Technological University, Singapore, in 1999. His research interests are in statistical and biomedical signal/image processing, and machine learning. Dr. Lin was the Editor-in-Chief of Multidimensional Systems and Signal Processing from 2011 to 2015, and has been in its editorial board since 1993. He was an Associate Editor of IEEE TRANSACTIONS ON CIRCUITS AND SYSTEMS-PART II: EXPRESS BRIEFS and the Subject Editor for the Journal of the Franklin Institute. He is a co-author of the 2007 Young Author Best Paper Award from the IEEE SIGNAL PROCESSING SOCIETY. He was a Distinguished Lecturer of the IEEE Circuits and Systems Society (CAS) from 2007 to 2008, and served as the Chair of IEEE CAS Singapore Chapter from 2007 to 2008, and in 2019.
\end{IEEEbiography} 

\end{document}